\documentclass[journal,twocolumn,compsoc,twoside]{IEEEtran}
\ifCLASSINFOpdf
   \usepackage[pdftex]{graphicx}
   \usepackage{amsfonts}
   \DeclareGraphicsExtensions{.pdf,.jpeg,.png}
   \usepackage[numbers,longnamesfirst,sort&compress]{natbib}
\usepackage{color}
\fi

\usepackage{url}

\usepackage{times}
\usepackage{epsfig}
\usepackage{graphicx}
\usepackage{amsmath}
\usepackage{amssymb}
\usepackage{epstopdf}
\usepackage{mathrsfs}

\usepackage{amsthm}

\usepackage{mathrsfs}
\usepackage{epstopdf}

\usepackage{threeparttable}
\usepackage{bm}

\newtheorem{prop}{\textbf{Proposition}}

\newtheorem{lemma}{\textbf{Lemma}}

\usepackage[linesnumbered,ruled,vlined]{algorithm2e}

\begin{document}
\newcommand{\tabincell}[2]{\begin{minipage}[#1][3.6em][c]{0.75\columnwidth}{#2}\end{minipage}}
\newcommand{\mylinebreak}{\vspace{0.16cm} \\}
\newcommand{\mean}{\mathrm{mean}}
\newcommand{\diag}[1]{\mathrm{diag}(#1)}

\title{Active Lighting Recurrence by Parallel Lighting Analogy for Fine-Grained Change Detection}

\author{Qian~Zhang$^*$,~\IEEEmembership{Student Member,~IEEE,}
       Wei~Feng$^{*\dagger}$,~\IEEEmembership{Member,~IEEE,}
       Liang~Wan,~\IEEEmembership{Member,~IEEE,}
       Fei-Peng~Tian,~\IEEEmembership{Student Member,~IEEE,}
       Xiaowei Wang,
       Ping~Tan,~\IEEEmembership{Member,~IEEE}
\IEEEcompsocitemizethanks{
	\IEEEcompsocthanksitem $^\dagger$W.~Feng is the corresponding author. Email: wfeng@ieee.org.
	\IEEEcompsocthanksitem $^*$Q.~Zhang \& W.~Feng are the joint first authors, who contribute equally to this work.  
	\IEEEcompsocthanksitem Q.~Zhang, W.~Feng and F.-P.~Tian are with the School of Computer Science and Technology, the College of Intelligence and Computing, Tianjin University, TianJin, 300350, China, and the Key Research Center for Surface Monitoring and Analysis of Cultural Relics (SMARC), State Administration of Cultural Heritage, China.
	\IEEEcompsocthanksitem L.~Wan is with the School of Computer Software, the College of Intelligence and Computing, Tianjin University, TianJin, 300350, China, and Tianjin Key Lab for Advanced Signal Processing, Civil Aviation University of China.
	\IEEEcompsocthanksitem X.~Wang is with Dunhuang Research Academy, Dunhuang, Gansu, China.
	\IEEEcompsocthanksitem P.~Tan is with the School of Computing Science, Simon Fraser University, Canada.}
}


\IEEEcompsoctitleabstractindextext{
\begin{abstract}
This paper studies a new problem, namely active lighting recurrence (ALR) that physically relocalizes a light source to reproduce the lighting condition from single reference image for a same scene, which may suffer from fine-grained changes during twice observations. ALR is of great importance for fine-grained visual inspection and change detection, because some phenomena or minute changes can only be clearly observed under particular lighting conditions. Therefore, effective ALR should be able to online navigate a light source toward the target pose, which is challenging due to the complexity and diversity of real-world lighting and imaging processes. To this end, we propose to use the simple parallel lighting as an analogy model and based on Lambertian law to compose an instant navigation ball for this purpose. We theoretically prove the feasibility, i.e., equivalence and convergence, of this ALR approach for realistic near point light source and small near surface light source. Besides, we also theoretically prove the invariance of our ALR approach to the ambiguity of normal and lighting decomposition. The effectiveness and superiority of the proposed approach have been verified by both extensive quantitative experiments and challenging real-world tasks on fine-grained change detection of cultural heritages. We also validate the generality of our approach to non-Lambertian scenes.
\end{abstract}

}

\maketitle

\IEEEdisplaynotcompsoctitleabstractindextext
\IEEEpeerreviewmaketitle

\section{Introduction}

\IEEEPARstart{L}{ighting} recurrence (LR) plays an important role in many computer vision applications, such as accurate surface and material acquisition~\cite{RL:ARFAHF2000,RL:CZHCSS2000,RL:OP2011}, cultural heritage imaging~\cite{RTI:PTM01,RTI:TAERIICLOUUNL11} and scene or image change surveillance~\cite{fcd:feng2015fine,Learn2fgcd}. Despite the variances and diversity of previous successful methods on this topic, including relighting~\cite{RL:RDLTG2015,RL:IRRFMCPDI2017}, reflectance transformation imaging (RTI)~\cite{RTI:NLCIRSR15,RTI:TAERIICLOUUNL11} and photometric stereo (PS)~\cite{PS:ABDEUPS16,PS:UPSSOUPCIB16}, most of them focus on \emph{synthetic (or virtual) lighting recurrence} (SLR), which passively take multi-illumination images as input to \emph{virtually} synthesize the target lighting condition.

In contrast, in this paper, we study \emph{active lighting recurrence} (ALR), a new problem aiming to \emph{physically} reproduce the target lighting condition, i.e., \emph{physically} relocalize the light source to exact the same pose of the reference one. This problem is originally motivated by an important and challenging real-world task, i.e., fine-grained change monitoring and measurement of cultural heritages for preventive conservation~\cite{pc_ref:HPPC13}. Generally, preventive conservation requires multi-observations for the same scene and massive-volume monitoring to analyze the causes of the deterioration of cultural heritages. Specifically, the major challenges of practical ALR are three-fold.
\begin{enumerate}
	\item \textbf{High accuracy requirement.} To support preventive conservation, we need to detect and measure very fine-grained changes occurred on cultural heritages, e.g., ancient murals, from complex image contents and various crumply and flaky deterioration patterns. To this end, physically reproducing camera pose and lighting condition as accurately as possible is critical to the quality of fine-grained change detection~\cite{our2018ACR, our2018ALR}.
	\item \textbf{Instant navigation requirement.} ALR is a dynamic process. The purpose of ALR is to produce reliable online navigation guidance and relocalize the light source to the target pose by a robotic platform or purely by hand. Note, many valuable cultural heritages inhabit unrestricted environments or limited working spaces, so instant navigation response is necessary for ALR problem, especially when the robotic platform cannot be used under some harsh working environments.
	\item \textbf{Accuracy vs. instantaneity.} Simple lighting models cannot accurately simulate reality but is usually very fast. In contrast, sophisticated lighting models usually need expensive optimization and a large number of images to guarantee the accuracy. That is, both of them cannot satisfy the requirements of instant navigation response, massive-volume and multi-observations monitoring. Hence, ALR must solve this challenge caused by accuracy and instantaneity requirements.
\end{enumerate}

Unfortunately, to the best of our knowledge, compared to relatively mature active camera relocalization (ACR)~\cite{dcr:feng2016,our2018ACR,ACRmobile,ACRdepth}, ALR is rarely studied, especially for fine-grained change detection (FGCD)~\cite{fcd:feng2015fine}. There is no mature solution in the literature solving all the above three challenges. First, the existing SLR methods mainly focus on the visual quality of relighted images and the strict physical correctness cannot be guaranteed, which may inevitably lead to lighting recurrence errors. Second, since most SLR methods usually need hundreds of multi-illumination images~\cite{RL:RDLTG2015,RTI:PTM01}, sophisticated lighting models and expensive optimizations~\cite{RL:RDLTG2015,PS:LSS2015} to pursue the recurrence accuracy, they cannot satisfy the requirements of real-time navigation response. See Fig.~\ref{fig:motif}(a), in real-world FGCD tasks, we usually get the reference observation by casting a particular near side lighting to highlight the rich 3D microstructures of the object. Fig.~\ref{fig:motif}(b) shows the observation by carefully and manually aligning lighting. Fig.~\ref{fig:motif}(c) shows the lighting recurrence image generated by a relighting method~\cite{RTI:TAERIICLOUUNL11} using 80 multi-illumination images. The 10-times magnified absolute differences between twice observations are shown in the bottom-right corner. Fig.~\ref{fig:motif}(e)--(i) show the local microstructures and the corresponding FGCD results~\protect\cite{fcd:feng2015fine}. We can see that without or inaccurate LR may generate two types of FGCD errors, that is, rich or particular 3D microstructures caused \emph{false alarms}, and different shading caused \emph{missing}, see Fig.~\ref{fig:motif}(e)--(h) and Fig.~\ref{fig:motif}(i), respectively.\footnote{The rich 3D microstructures of the object under different near side lightings have different shadows, shadings and specular spots. All of them can cause great false alarmed changes. In contrast, some real changes cannot be clearly observed under the lighting whose direction is close to the normal of the surface where the changes occur.} Both of them significantly harm the FGCD accuracy. Besides, the relighting method~\cite{RTI:TAERIICLOUUNL11} needs at least 30 minutes (including the time of images capturing, lighting calibration, lighting calculation and re-rendering) for once lighting recurrence, which clearly cannot satisfy the massive-volume and multi-observations monitoring requirements.

In this paper, we propose an effective ALR method for reproducing the lighting condition of a reference image. We use the simple parallel lighting (PL) as an analogy model to calculate the reference and current lighting conditions from a small amount of in-situ captured side lighting images (a dozen). Then we compose a navigation ball with two spherical isointensity circles (SICs) indicating the reference and current poses to generate reliable and real-time ALR navigation guidance. Theoretically, We find that the \emph{analogy parallel lighting} (apl) based ALR (ALR-apl) is equivalent to the one based on more sophisticated lighting models, e.g., near point lighting (NPL) model and small near surface lighting (sNSL) model. 
Besides, we also prove the invariance of the proposed ALR method to the normal \& lighting decomposition ambiguity. As shown in Fig.~\ref{fig:motif}(d)(e)--(i) and our extensive experiments, the proposed ALR works well for both Lambertian and non-Lambertian scenes, and can significantly improve FGCD performance. Our ALR method only needs less than 3 minutes for once lighting recurrence (including the time of in-situ image capturing, navigation generation and light source adjustment). 
Part of our findings and results have been published in~\cite{our2018ALR}.

\begin{figure}[t]
	\begin{center}
		\includegraphics[width=1\linewidth]{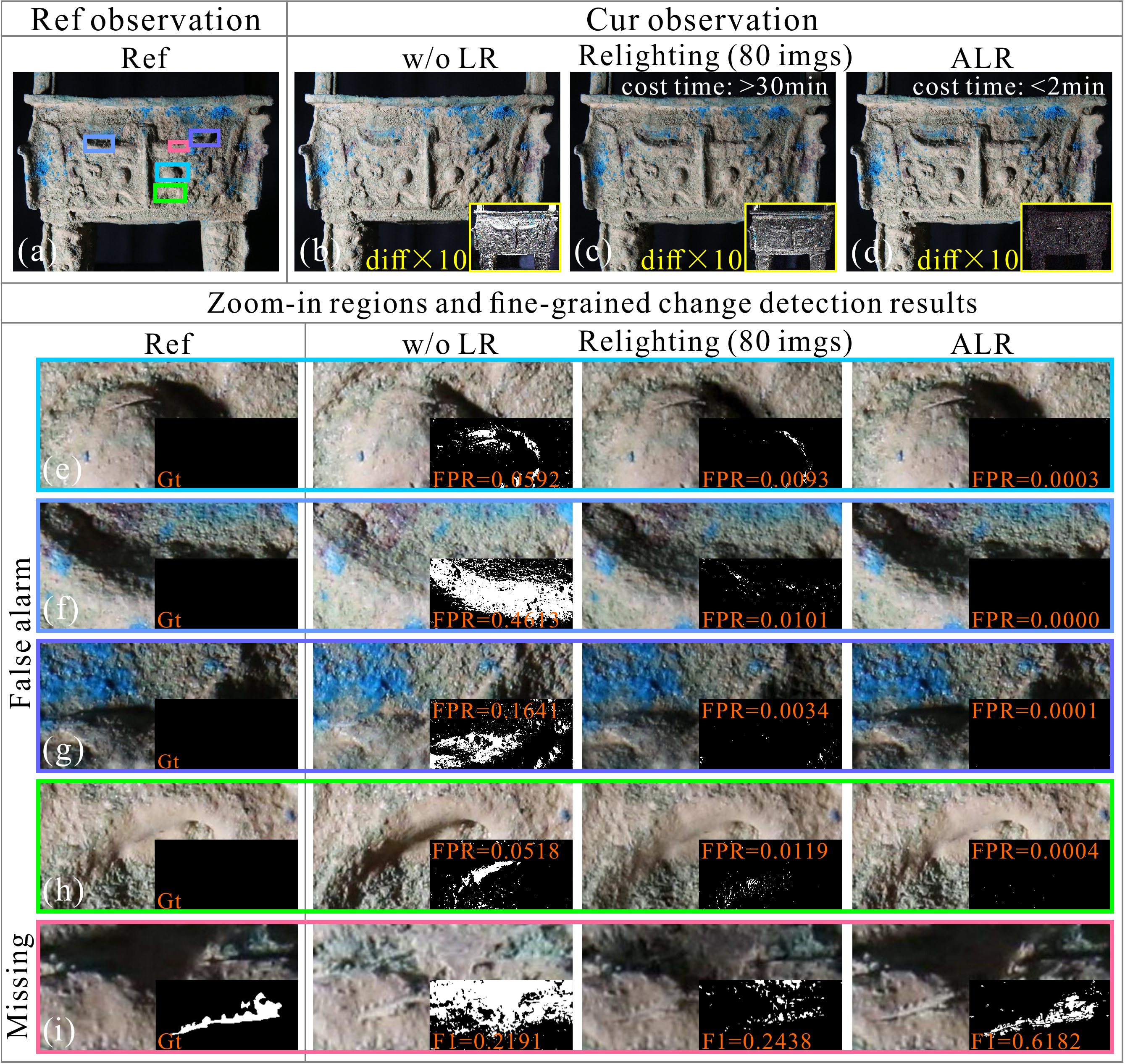} \vspace{-20pt}
	\end{center}
	\caption{\small{Motivation and importance of active lighting recurrence for fine-grained change detection. See text for details.}}
	\label{fig:motif} \vspace{-5pt}
\end{figure}

\section{Related Work}

\subsection{RTI and CHI}
Reflectance transformation imaging (RTI) aims to re-render a scene under arbitrary lighting direction, by sampling images under known lighting directions. RTI is widely used in cultural heritage imaging (CHI), using image-based methods to effectively capture and visualize the geometry and material of cultural heritages. Polynomial basis function (PTM)~\cite{RTI:PTM01} and hemispherical harmonics~\cite{RTI:TAERIICLOUUNL11} are two representative RTI methods, which represent the reflectance function of a scene by lower-order basis functions. A survey about computational imaging for cultural heritage can be accessed~\cite{RTI:RDSI2016}. 

\subsection{ Photometric stereo}
Photometric stereo (PS) focuses on acquiring surface shape and reflectance from multi-illumination images. Classical PS methods~\cite{PS:PMDSOMI80, PS:DSRMI80} solve reflectance and normal under the assumption of Lambertian surface and infinitely far lights. Recent PS methods focus on non-Lambertian surfaces and separately parameterize specularity and diffusion reflections~\cite{PS:ABRMRIA12, PS:PSCBRIS14, PS:ACFSUPSDM12}. For uncalibrated lighting conditions, PS needs also determine the exact lighting directions~\cite{PS:PSGUL07, PS:ACFSUPSDM12, PS:UPSSOUPCIB16}. A thorough survey on uncalibrated and non-Lambertian photometric stereo can be found in~\cite{PS:ABDEUPS16}. Besides, how to relax the far-light assumption is also widely studied in~\cite{PS:ANAPSLIRCL11, PS:PSNPLASMD15, PS:IRLSUSM14}.

\subsection{Image-based relighting} 
Image-based relighting aims to calculate the light transport matrix which stores the relation between the intensity of each pixel and different lighting conditions. Given an arbitrary lighting condition, a new image can be easily generated from the light transport matrix. Classical methods generally use a brute-force solution to measure the entries of light transport matrix~\cite{RL:ARFAHF2000,RL:ADLS2005}. An early survey about relighting can be found in~\cite{RL:ASIRT2007}. Later, for reducing the number of captured images, sparse representations of light transport matrix have been studied in many works, e.g., introducing compressive sensing~\cite{RL:CLTS2009}, dual photography~\cite{RL:CDP2009} and appropriate illumination patterns~\cite{RL:FDALTSVI2012}. Furthermore, many methods~\cite{RL:FBS2005, RL:OCFFLTA2010, RL:ASRF2007} exploit the data coherence to reconstruct the light transport matrix with fewer images. Recently, neural networks flourished in relighting fields~\cite{RL:FGIDHF2009, RL:GIRRF2013, RL:RDLTG2015, RL:IRRFMCPDI2017,ye2018single}.

\begin{figure*}[t]
	\begin{center}
		\includegraphics[width=1.0\linewidth]{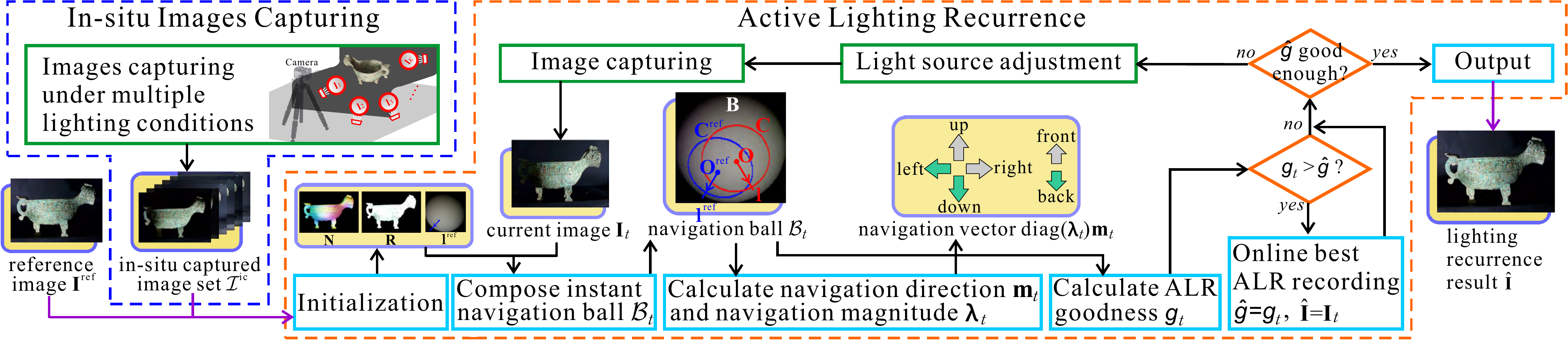} \vspace{-0.6cm}
	\end{center}
	\caption{\small {Overall framework of the proposed active lighting recurrence (ALR) approach that works well for realistic near point light source or small near surface light source. The green and light-blue blocks indicate pose-adjustment (or image-capturing) and online calculation processes, respectively. See text for details.}}
	\label{fig:framework} \vspace{-0.0cm}
\end{figure*}

\section{Problem Formulation}
Let ${\bf{I}}^{\rm{ref}},{\bf{I}} \in \mathbb{R}^{P}$ be the reference and current observations, where $P$ being the pixel number. Let $\rm{L}^{ref}$ and ${\rm{L}}$ be the reference and current lighting conditions. Then we have ${\bf{I}}^{\rm{ref}} = {\rm{F}}_{\rm{L}}( {\rm{L}}^{\rm ref}, {\mathcal{S}}_{\tau})$ and ${\bf{I}} = {\rm{F}}_{\rm{L}}( {\rm{L}}, {\mathcal{S}}_{\tau + \vartriangle \tau})$, where ${\mathcal{S}}_{\tau}$ and ${\mathcal{S}}_{\tau+\vartriangle \tau}$ denote the scene (e.g., reflectance, normal, specular regions, if any) at reference time $\tau$ and current time $\tau + \vartriangle\tau$ respectively, ${\rm{F}}_{\rm{L}}(\cdot)$ indicates the real lighting model, which is determined by scene and lighting condition. Note, the scene may occur some fine-grained changes during time interval $\vartriangle\tau$. Lighting recurrence (LR) aims to reproduce the reference lighting condition and recreate a current image which is similar (except the change region) to the reference one. We formulate the lighting condition estimation in LR problem as
\begin{equation} \label{eq:formulation_LR}
\begin{array}{rl}
{\hat {\rm L}} =\displaystyle \mathop{\arg\min}_{\rm L} \| {\rm{F}}_{\rm{L}}( {\rm{L}}, {\mathcal{S}}_{\tau + \vartriangle \tau}) - {\rm{F}}_{\rm{L}}( {\rm{L}}^{\rm ref} ,{\mathcal{S}}_{\tau}) \|_{\rm{F}}^2,
\end{array}
\end{equation}
where $\| \cdot \| _{\rm{F}}$ is the Frobenius norm. 

As we mentioned above, synthetic lighting recurrence (SLR) usually use a specific lighting model to denote ${\rm{F}}_{\rm{L}}(\cdot)$ and generate the relighted image by ${\hat{\bf I}} = {\rm{F}}_{\rm{L}}( {\hat {\rm L}}, {\mathcal{S}}_{\tau + \vartriangle \tau})$. However, SLR methods cannot guarantee strict physical correctness, which may harm the relighting performance and reduce the FGCD accuracy. Different with SLR, we focus on the new problem, i.e., active lighting recurrence (ALR), which aims to physically reproduce the lighting condition. Since lighting condition $\rm {L}$ is determined by the extrinsic parameter (pose) ${{\bm{\rho}}}$ and intrinsic parameter (radiation power, color temperature, intensity distribution) ${\bf{\Theta}}$ of the light source, let $\rm{L} = {\mathcal{L}}({\rm{\bm{\rho}}},{\bf{\Theta}})$ and $\rm{L}^{ref} = {\mathcal{L}}({\rm{\bm{\rho}}}^{ref},{\bf{\Theta}}^{ref})$, Eq.~(\ref{eq:formulation_LR}) can be rewritten as
\begin{equation} \label{eq:formulation_gALR}
\begin{array}{rl}
\langle {\hat{\rm{\bm{\rho}}}, \hat{\bf{\Theta}}} \rangle =\displaystyle \mathop{\arg\min}_{{\langle{{\rm{\bm{\rho}}}, {\bf{\Theta}}}\rangle}} & \| {\rm{F}}_{\rm{L}}( {\mathcal{L}}({\rm{\bm{\rho}}}, {\bf{\Theta}}), {\mathcal{S}}_{\tau + \vartriangle \tau})\\& - {\rm{F}}_{\rm{L}}( {\mathcal{L}}({\rm{\bm{\rho}}}^{\rm{ref}}, {\bf{\Theta}}^{\rm{ref}}) ,{\mathcal{S}}_{\tau}) \|_{\rm{F}}^2.
\end{array}
\end{equation}
Eq.~(\ref{eq:formulation_gALR}) aims to reproduce both the intrinsic and extrinsic parameters of light source. In fact, it is hard even impractical for us to make the intrinsic parameters of two different kinds of light sources be consistent. Hence, we assume the light source is the same during twice observations, i.e., ${\bf{\Theta}}^{\rm{ref}} = {\bf{\Theta}}$. Then the aim of ALR becomes to relocalize the extrinsic parameter, i.e., light source pose. Besides, since the coordinate translation between light source and camera is uncalibrated, we cannot directly move the light source to the target pose $\hat{\rm{\bm{\rho}}}$ by one-shot adjustment, even if using an accurate robotic platform. To solve this problem, we employ the \emph{progressive adjustment strategy}, which has been successfully used in active camera relocalization (ACR)~\cite{dcr:feng2016,our2018ACR}.
Hence, we formulate the ALR problem as a \emph{dynamic process} that iteratively calculates the navigation guidance and adjusts light source pose, 
\begin{equation} \label{eq:formulation_ALR}
\begin{array}{rl}
{\hat{\rm{\bm{\rho}}}} &= \displaystyle \mathop{\arg\min}_{\bm{\rho}}  \| {\rm{F}}_{\rm{L}}( {\mathcal{L}}({\rm{\bm{\rho}}} \mid {\bf{\Theta}}), {\mathcal{S}}_{\tau + \vartriangle \tau}) \\& {\ \ \ \ \ \ \ \ \ \ \ }- {\rm{F}}_{\rm{L}}( {\mathcal{L}}({\rm{\bm{\rho}}}^{\rm{ref}} \mid {\bf{\Theta}}) ,{\mathcal{S}}_{\tau}) \|_{\rm{F}}^2 \\
&=\displaystyle \mathop{\arg\min}_{{{\rm{\bm{\rho}}}}}\| {\rm{F}}_{\rm{L}}( {\mathcal{L}}({\rm{\bm{\rho}}} \mid {\bf{\Theta}}), {\mathcal{S}}_{\tau + \vartriangle \tau}) - {\bf{I}}^{\rm{ref}} \|_{\rm{F}}^2,
\end{array}
\end{equation}
\begin{equation} \label{eq:formulation_strategy}
\begin{array}{rl}
&{{\bm{\rho}}_{t+1}} = {{\bm{\rho}}_{t}} + {\rm{diag}}({\bm{\lambda}}_{t}){\bf{m}}_{t}, \\
&{\hat{\rm{\bm{\rho}}}} = {\lim \limits_{t \to \infty}} {{\bm{\rho}}_{t}},
\end{array}
\end{equation}
where ${\bf{m}}_{t}$ and ${\bm{\lambda}}_{t}$ indicate the ALR navigation direction and magnitude for $t$th light source pose adjustment, ${\rm diag}(\cdot)$ is the diagonalization of a vector. Note, Eq.~(\ref{eq:formulation_ALR}) defines the ALR target, while Eq.~(\ref{eq:formulation_strategy}) is the progressive ALR strategy, i.e., we physically adjust the current light source pose by $\Delta {\bm{\rho}}_{t} = {\rm{diag}}({\bm{\lambda}}_{t}){\bf{m}}_{t}$ in $t$th adjustment. Hence, the convergence and goodness of an ALR approach relies on the correctness of ${\bm{\lambda}}_{t}$ and ${\bf{m}}_{t}$, and $\displaystyle \lim_{t \to \infty} {\rm{diag}}({\bm{\lambda}}_{t}){\bf{m}}_{t} = \mathbf{0}$.

There exist two challenges to solve above ALR problem. First, for estimating the light source pose, we need to solve the inverse problem of the lighting model ${\rm{F}_L}(\cdot)$. Although we can use gradient descendant method to solve Eq.~(\ref{eq:formulation_ALR}) based on a sophisticated and realistic lighting model, it cannot satisfy the requirement of real-time navigation response. Second, since the scene ${\mathcal{S}}_{\tau + \vartriangle \tau}$ is unknown, the inverse problem is ill-posed and there may exists an ambiguity for scene and lighting decomposition. In this paper, we propose an \emph{analogy parallel lighting} (apl) based ALR (ALR-apl) to perfectly solve these problems and we theoretically prove our ALR-apl is equivalent to the one based on more sophisticated lighting models.



\section{ALR under Parallel Lighting}\label{sec:ALR-PL}
We use the simple parallel lighting as an analogy model. See Fig.~\ref{fig:threemodel}(a), let ${\bf{l}} \in \mathbb{R}^{3}$ be the parallel lighting vector, whose magnitude and direction indicate the lighting strength and direction, respectively. 
Then we have ${\mathcal{L}}({\rm{\bm{\rho}}} \mid {\bf{\Theta}})=\bf{l}$ and ${\mathcal{S}_{\tau + \vartriangle \tau}} = \{ \bf{R}, \bf{N} \}$, with ${\bf{R}} \in \mathbb{R}^{P}$ and ${\bf{N}} \in \mathbb{R}^{P \times 3}$ being the scene reflectance and normal. Thus, ALR problem (Eqs.~(\ref{eq:formulation_ALR})--(\ref{eq:formulation_strategy})) can be reduced to a much simpler ALR-apl problem
\begin{equation} \label{eq:ALR-PL} 
\begin{array}{rl}
{\hat{\bf{l}}} =\displaystyle \mathop{\arg\min}_{\bf{l}} \| {\bf{R}} \circ {\bf{Nl}} - {\bf{I}}^{\rm{ref}} \|_{\rm{F}}^2,
\end{array}
\end{equation}
\vspace{-0.3cm}
\begin{equation} \label{eq:ALR-PL_strategy}
\begin{array}{rl}
&{{\bf{l}}_{t+1}} = {{\bf{l}}_{t}} + {\rm{diag}}({\bm{\lambda}}_{t}){\bf{m}}_{t}, \\
&{\hat{\rm{\bf{l}}}} = {\lim \limits_{t \to \infty}} {{\bf{l}}_{t}},
\end{array}
\end{equation}
where $\circ$ indicates element-wise multiplication. 

Thanks to the simplicity of ALR-apl model, it is possible to online calculate both the navigation direction ${\bf{m}}_{t}$ and magnitude ${\bm{\lambda}}_{t}$ from the reference and current image, ${\bf{I}}^{\rm{ref}}$ and ${\bf{I}}_t$. Fig.~\ref{fig:framework} shows the working flow of the ALR-apl approach. Specifically, to get a reliable initialization, we first roughly capture $K$ different side lighting images to form the in-situ captured image set ${\mathcal{I}}^{\rm{ic}}$. From ${\mathcal{I}}^{\rm{ic}}$ and $\bf{I}^{\rm{ref}}$, we obtain scene normal ${\bf{N}}$, reflectance ${\bf{R}}$ and reference lighting vector ${\bf{l}}^{\rm{ref}}$. Then, in each ALR-apl iteration, we compose an instant navigation ball $\mathcal{B}$ and online calculate the navigation direction ${\bf{m}}_{t}$ and magnitude ${\bm{\lambda}}_{t}$ for light source adjustment. We analyze the convergence and efficacy of this process at last. 

\subsection{Initialization}
By parallel lighting analogy, we have
\begin{equation} \label{eq:S=NL}
{\bf{S}} = {\bf{N}} {\bf{l}},
\end{equation}
where ${\bf{S}} \in \mathbb{R}^{P}$ is shading image that can be obtained by disambiguated intrinsic image decomposition~\cite{CC:zqtpdqslwels12}, by solving $\bf{I} = \bf{R} \circ \bf{S}$. Clearly, from ${\mathcal{I}}^{\rm{ic}}$ and $\bf{I}^{\rm{ref}}$, we can obtain $K+1$ shading images and we use them to calculate the scene normal ${\bf{N}}$ and reference (target) lighting vector ${\bf{l}}^{\rm{ref}}$ (i.e., $\hat {\bf{l}}$ in Eq.~(\ref{eq:ALR-PL})) via a fast state-of-the-art uncalibrated photometric stereo algorithm, LDR~\cite{PS:ACFSUPSDM12}. Considering that the change occurred in scene is tiny, here we omit the influence of time interval $\vartriangle \tau$ on scene, i.e., ${\mathcal{S}}_{\tau} \approx {\mathcal{S}}_{\tau + \vartriangle \tau} = \{ \bf{R}, \bf{N} \}$ \footnote{The simplification about scene information is only for estimating reference lighting condition ${\bf{l}}^{\rm{ref}}$, since both ${\bf{I}}^{\rm{ref}}$ and ${\bf{I}}$ are all real images captured by camera, the change can be reproduced in the currently captured images.}. In practice, we only need $K \approx 12$ side lighting images for initialization.


\subsection{ Instant navigation ball $\mathcal{B}$}
Given $\bf{N}$ and $\bf{R}$, the lighting vector ${\bf{l}}_t$ corresponding to the current image ${\bf{I}}_t$ can be easily obtained by solving Eq.~(\ref{eq:S=NL}) in closed-form solution. To eliminate scene dependency in ALR process, we online render both current and reference lightings, ${\bf{l}}_t$ and ${\bf{l}}^{\rm{ref}}$, onto a unit sphere, rather than on the real scene normal $\bf{N}$. Specifically, let ${\bf{B}}^{\rm{ref}}$ and ${\bf{B}}_{t}$ be the reference and currently rendered images of the unit sphere, i.e., ${\bf{B}}^{\rm{ref}} = {\bf{N}}^{\rm{s}} {\bf{l}}^{\rm{ref}}$ and ${\bf{B}}_{t} = {\bf{N}}^{\rm{s}} {\bf{l}}_{t}$, where ${\bf{N}}^{\rm{s}}$ is the sphere normal which is known. From the Lambertian law, we can easily obtain the following proposition about the spherical isointensity sets ${\bf{C}}^{\rm{ref}}$ and ${\bf{C}}_{t}$, formed by rendered pixels with some particular intensity value, e.g., the median of ${\bf{B}}^{\rm{ref}}$.

\begin{prop}[SICs \& shading equivalence] \label{prop:ALR_apl}
	With analogy parallel lighting and Lambertian law, given an arbitrary lighting vector $\bf{l}$, the spherical isointensity set $\bf{C}$ always forms a circle under the view of lighting direction, which can be named as spherical isointensity circle~(SIC). Iff the reference and current SICs, ${\bf{C}}^{\rm{ref}}$ and ${\bf{C}}_{t}$ coincide completely, the reference and current images, ${\bf{I}}^{\rm{ref}}$ and ${\bf{I}}_{t}$, are the same.
\end{prop}
\begin{proof} SICs: For Eq.~(\ref{eq:S=NL}), we replace $\bf{N}$ by a sphere normal $\bf{N}^{\rm{s}}$. For arbitrary point $p$, we have
	\begin{equation} \label{eq:PLM_1}
	{{\bf{B}}_p} = {\bf{N}}_p^{\rm{s}} {{\bf{l}}}={\left\| {\bf{l}} \right\|}{\rm{ cos\ }} {\alpha_{p}},
	\end{equation}
	where $\alpha_{p}$ is the included angle of ${\bf{N}}_p^{\rm{s}}$ and $\bf{l}$. Given two points $p$ and $q$ which satisfy ${{\bf{B}}_p} = {{\bf{B}}_q}$ and $\alpha_{p},\alpha_{q} \in [ 0 ,\frac{\pi}{2} ]$, we have
	\begin{equation} \label{eq:PLM_2}
	\begin{aligned}
	{\rm{cos\ }}\alpha_{p}={\rm{  cos\ }}\alpha_{q} \Leftrightarrow  \alpha_{p} = \alpha_{q}.
	\end{aligned}
	\end{equation}
	It means that if any two points on the rendered image $\bf{B}$ have the same intensity, the normals of the two points have the same angle with $\bf{l}$. In other words, the points having identical intensity value on $\bf{B}$, i.e., the spherical isointensity set $\bf{C}$ forms a circle under the view of the parallel lighting direction.
	
	Shading equivalence: We first prove the sufficient condition. For each pixel $p$ in $\bf{C}^{\rm{ref}}$ (and $\bf{C}$) on $\bf{B}^{\rm{ref}}$ (and $\bf{B}$), we have ${{\bf{B}}^{\rm{ref}}_{p}}{\rm{ = }}{\bf{N}}_{p}^{{\rm{s}}}  {\bf{l}}^{\rm{ref}}$ and ${{\bf{B}}_{p}}{\rm{ = }}{\bf{N}}_{p}^{{\rm{s}}}  {\bf{l}}$. Considering ${\bf{C}}^{\rm{ref}} = \bf{C}$, so ${\bf{B}}^{\rm{ref}}_{p}={\bf{B}}_{p}$. Therefore, only if ${\bf{C}}^{\rm{ref}}$ contains more than three spherical points, the linear system ${\bf{N}}_{p}^{{\rm{s}}}{\bf{l}}^{\rm{ref}} = {\bf{N}}_{p}^{{\rm{s}}}{\bf{l}}$ leads to ${\bf{l}}^{\rm{ref}}={\bf{l}}$. Furthermore, since scene normal and reflectance remain the same, the real captured images are fully determined by the lighting condition. Thus, we have ${\bf{I}}^{\rm{ref}}={\bf{I}}$. The proof of necessary condition is similar to the sufficient condition and we do not provide the details. 
\end{proof}

\begin{figure*}[t]
	\begin{center}
		\includegraphics[width=0.9\linewidth]{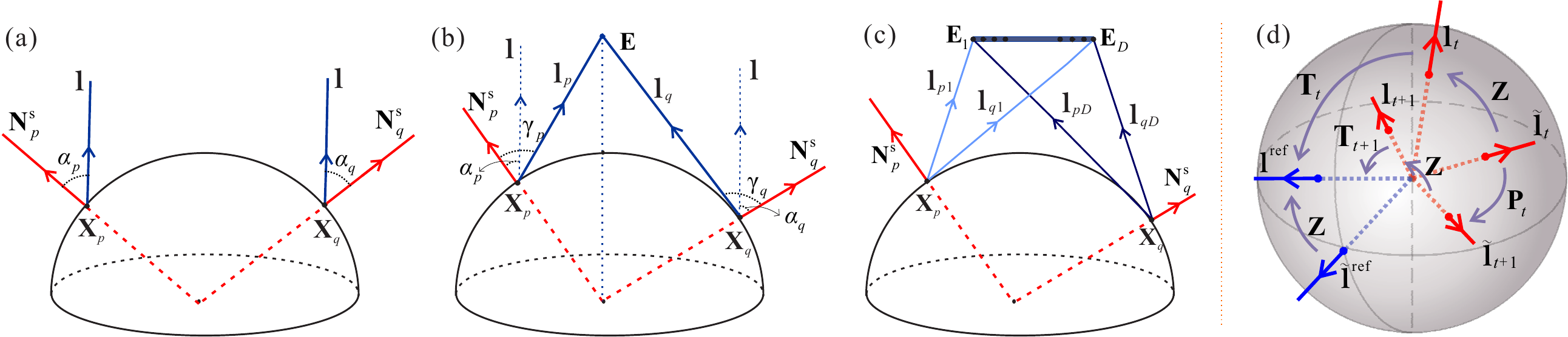}\vspace{-0.3cm}
	\end{center}
	\caption{\small { (a)~The parallel lighting model; (b)~The near point lighting model; (c)~The small near surface lighting model; (d)~The illustration of the ambiguity matrix $\bf{Z}$ of $\mathbf{N}$ \& $\mathbf{l}$ decomposition for Lemma~\ref{lemma:xy-axis}, see text for details. }}
	\label{fig:threemodel} \vspace{-0cm}
\end{figure*}

Proposition~\ref{prop:ALR_apl} means that ALR is finished if we can make the reference and current SICs, ${\bf{C}}^{\rm{ref}}$ and ${\bf{C}}_{t}$ coincide by adjusting light source pose. As shown in Fig.~{\ref{fig:balltransformation}}, we dynamically compose an \emph{instant navigation ball} $\mathcal{B}_t$ to provide effective instant ALR navigation for adjusting light source,
\begin{equation} \label{eq:PLM_navigationball}
\mathcal{B}_t = \{ {\bf{B}}_t, {\bf{C}}_t, {\bf{O}}_t, {\bf{C}}^{\rm{ref}}, {\bf{O}}^{\rm{ref}} \},
\end{equation}
where ${\bf{B}}_t$ is the current rendered image, ${\bf{O}}^{\rm{ref}}\triangleq[1,\theta^{\rm ref}, \varphi^{\rm ref}]^{\rm T}$ and ${\bf{O}}_{t}\triangleq[1,\theta_t, \varphi_t]^{\rm T}$ are the center coordinates of ${\bf{C}}^{\rm{ref}}$ and ${\bf{C}}_{t}$, respectively. Essentially, the navigation ball $\mathcal{B}_t$ provides a visual navigation guidance in $t$th iteration, which can be used to calculate the navigation direction and magnitude.

\subsection{Online calculation of navigation direction ${\bf{m}}_{t}$}

\begin{figure}[t] \vspace{-0cm}
	\begin{center}
		\includegraphics[width=1.0\linewidth]{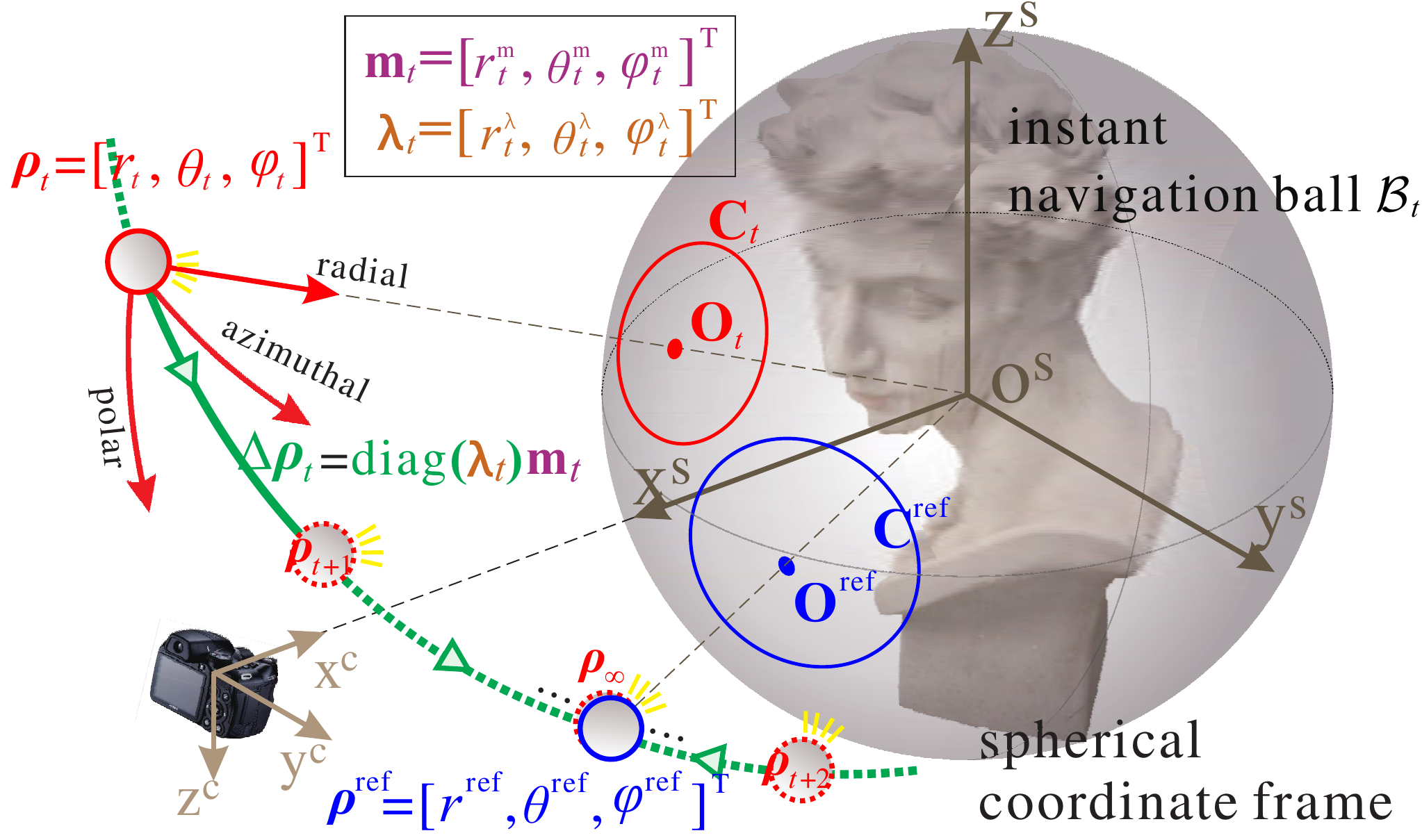}  \vspace{-0.6cm}
	\end{center}
	\caption{\small {The instant navigation ball $\mathcal{B}_t$, reference and current SICs ${\bf{C}}^{\rm{ref}}$ (blue) \& ${\bf{C}}_{t}$ (red) in the spherical coordinate frame, the parameterization of light source pose ${{\bm{\rho}}_{t}} = [r_t, \theta_t, \varphi_t]^{\mathrm{T}}$ and ALR increment $\Delta {\bm{\rho}}_{t} = {\rm{diag}}({\bm{\lambda}}_{t}){\bf{m}}_{t}$, and light source ALR adjustment trajectory (green line).  }}
	\label{fig:balltransformation} \vspace{-0cm}
\end{figure}

As shown in Fig.~{\ref{fig:balltransformation}}, since $\theta_t$ and $\varphi_t$ encode the lighting direction of light source in the azimuthal and polar axes respectively, an analogy parallel light source pose ${{\bm{\rho}}_{t}}$ can be expressed by vector $[r_t, \theta_t, \varphi_t]^{\mathrm{T}}$ in the spherical coordinate frame, with $r_t$ being the distance between the light source and the coordinate frame center. 
Therefore, we can effectively calculate the ALR navigation direction ${\bf{m}}_{t}$ in radial, azimuthal and polar axes according to $\mathcal{B}_t$,
\begin{equation} \label{eq:navcom}
{\bf{m}}_{t} = [r_t^{\rm m}, \theta_t^{\rm m}, \varphi_t^{\rm m}]^{\mathrm{T}} = {\rm{sgn}} ([A^{\rm{ref}}, \theta^{\rm{ref}}, \varphi^{\rm{ref}}]^{\mathrm{T}} - [A_t, \theta_t, \varphi_t]^{\mathrm{T}} ),
\end{equation}
where $A^{\rm{ref}} = \mathrm{A}({\bf{C}}^{\rm{ref}})$ and $A_t = \mathrm{A}({\bf{C}}_{t})$ are the areas of SICs ${\bf{C}}^{\rm{ref}}$ and ${\bf{C}}_{t}$, respectively, $\mathrm{A}(\cdot)$ denotes region area, ${\rm{sgn}}(\cdot)$ is the sign function. In fact, the area of SIC encodes the distance information between light source and coordinate frame center. As illustrated by Fig.~{\ref{fig:balltransformation}}, ${\bf{m}}_{t}$ reflects the positive ($+1$) or negative ($-1$) ALR directions respectively along the three axes of spherical coordinate, for the apl model.

\subsection{Online calculation of navigation magnitude ${\bm{\lambda}}_{t}$} 
With real-time calculation of current navigation direction ${\bf{m}}_{t}$, we can establish a manual control loop with $[A^{\rm{ref}}, \theta^{\rm{ref}}, \varphi^{\rm{ref}}]^{\mathrm{T}}$ as desired set-point (SP) and $[A_t, \theta_t, \varphi_t]^{\mathrm{T}}$ being the measured process variable (PV). For many real-world fine-grained change detection tasks, manual ALR, or ALR with hand (ALR\_H), is reliable due to its great portability. 

Nevertheless, as shown in Fig.~{\ref{fig:balltransformation}}, we can also effectively calculate the ALR navigation magnitude ${\bm{\lambda}}_{t}$ via \emph{bisection approaching}, which together with ${\bf{m}}_{t}$ can enable an \emph{automatic ALR} process, with the help of a robotic platform (ALR\_R). Specifically, given an initial ${\bm{\lambda}}_{0} = [r_0^{\rm{\lambda}}, \theta_0^{\rm{\lambda}}, \varphi_0^{\rm{\lambda}}]^{\rm{T}}$, where $r_0^{\rm{\lambda}}$, $\theta_0^{\rm{\lambda}}$ and $\varphi_0^{\rm{\lambda}}$ indicate the initial light source adjustment magnitude in radial, azimuthal and polar axes, respectively. In $t$th ($t\geq1$) ALR iteration, ${\bm{\lambda}}_{t}=[r_{t}^{\rm{\lambda}},\theta_{t}^{\rm{\lambda}},\varphi_{t}^{\rm{\lambda}}]^{\rm{T}}$ satisfies
\begin{equation} \label{eq:bisectstrategy}
\begin{aligned}
a_{t}^{\rm{\lambda}} = \left\{
\begin{array}{rcl}
\frac{1}{2} a_{t-1}^{\rm{\lambda}}      &      & {a_{t}^{\rm{m}} a_{t-1}^{\rm{m}} < 0},\\
\mu a_{t-1}^{\rm{\lambda}}    &      & \rm{otherwise},\\
\end{array} \right.
\end{aligned}
\end{equation}
where $a \in \{r, \theta, \varphi\}$ denotes the three independent spherical axes, $\mu$ is the speed-up rate of navigation magnitude and is empirically set as $1.2$ in our experiments. With Eq.~(\ref{eq:bisectstrategy}), we can efficiently obtain the ALR navigation vector, i.e., ALR increment $\Delta {\bm{\rho}}_{t} = {\rm{diag}}({\bm{\lambda}}_{t}){\bf{m}}_{t}$, which is directly applied on the robotic platform to conduct the $t$th ALR adjustment.  

\subsection {The algorithm and implementation details}

\begin{algorithm}[ht] 	
	\small
	\caption{Active Lighting Recurrence} \label{alg:practical}
	\KwIn{Normal $\bf{N}$, reflectance $\bf{R}$, reference lighting vector ${\bf{l}}^{\rm{ref}}$ and stopping threshold $\eta$}
	\KwOut{$\hat {\bf{I}}$}
	Initialization: $\hat g=0$, $g_{0}=0$, $t=0$\;
	\While{$g_t \leq \eta$}
	{
		$t++$\;
		Capture current image ${\bf{I}}_t$\;
		Calculate current lighting vector ${\bf{l}}_t$ by Eq.~(\ref{eq:S=NL})\;
		Compose navigation ball $\mathcal{B}_t$ according to Eq.~({\ref{eq:PLM_navigationball}})\;
		Calculate IoU $g_t$ by Eq.~({\ref{eq:IOU}})\;
		\If{$g_t > \hat g$}
		{
			$\hat g=g_t, ~~ {\hat {\bf{I}}}={\bf{I}}_t$\;
		}
		Calculate navigation command ${\bf{m}}_t$ by Eq.~(\ref{eq:navcom})\;
		Calculate navigation magnitude $\bm{\lambda}_t$ by Eq.~(\ref{eq:bisectstrategy})\;
		Adjust light source pose by ${\rm{diag}}({\bm{\lambda}}_t){\bf{m}}_t$;
	}
	return $\hat {\bf{I}}$.
\end{algorithm}

Algorithm~\ref{alg:practical} shows the detailed working flow of our ALR-apl approach. Specifically, in $t$th iteration, we first compose the navigation ball $\mathcal{B}_t$ according to ${\bf{l}}^{\rm{ref}}$ and ${\bf{l}}_t$, then calculate the navigation direction ${\bf{m}}_t$ and navigation magnitude $\bm{\lambda}_t$. After that, we adjust light source pose by ${\rm{diag}}({\bm{\lambda}}_t){\bf{m}}_t$. This iterative adjustment process is terminated by an ALR goodness $g_t$ that measures the recurrence accuracy by the overlap ratio, i.e., the Intersection-over-Union (IoU) of ${\bf{C}}^{\rm{ref}}$ and ${\bf{C}}_t$,
\begin{equation} \label{eq:IOU}
\begin{aligned}
g_t = \frac{\mathrm{A}({\rm{reg}}({\bf{C}}_{t}) \cap {\rm{reg}}(\bf{C}^{\rm{ref}}))}{\mathrm{A}({\rm{reg}}({\bf{C}}_{t}) \cup {\rm{reg}}(\bf{C}^{\rm{ref}}))},
\end{aligned}
\end{equation}
where ${\rm{reg}}({\bf{C}}_t)$ and ${\rm{reg}}(\bf{C}^{\rm{ref}})$ indicate the regions inclosed by ${\bf{C}}_t$ and $\bf{C}^{\rm{ref}}$ respectively, $\rm A(\cdot)$ is region area. During ALR, we record the largest $g_t$ and the corresponding current image by $\hat g$ and $\hat {\bf{I}}$, respectively. Once $g_t \leq \eta$, $\hat {\bf{I}}$ is just the final lighting recurrence result. We set $\eta$ to $0.98$ and we use the median of reference rendered image ${\bf{B}}^{\rm{ref}}$ to form ${\bf{C}}^{\rm{ref}}$ and $\bf{C}$. We empirically set $r_0^{\rm{\lambda}}=\theta_0^{\rm{\lambda}}=\varphi_0^{\rm{\lambda}}=5$mm in our experiments.

\subsection{Convergence analysis}\label{sec:ALRconvergence}

\begin{figure}[t] \vspace{-0cm}
	\begin{center}
	\includegraphics[width=1.0\linewidth]{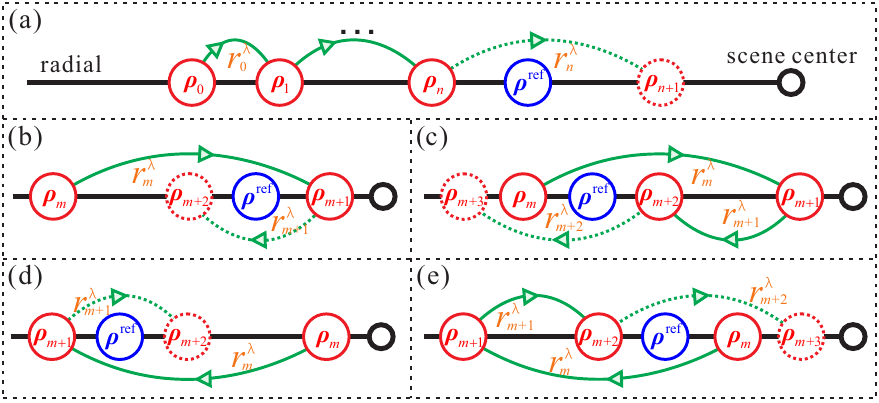}  \vspace{-0.6cm}
	\end{center}
	\caption{\small {The light source pose adjustment in radial-axis during ALR. (b) (or (c)) and (d) (or (e)) happen alternately after (a). According this observation, we can prove the convergence of pose adjustment in radial-axis. See text for details.}}
	\label{fig:ALRconvergence} \vspace{-0cm}
\end{figure}

Eq.~(\ref{eq:bisectstrategy}) indeed defines a \emph{bisection approaching strategy}. We can see that the light source pose adjustments in the three spherical axes are independent of one another. In $t$th iteration, we physically adjust the light source pose by ${\rm{diag}}({\bm{\lambda}}_{t}){\bf{m}}_{t}$, which is calculated by Eq.~(\ref{eq:navcom}) and Eq.~(\ref{eq:bisectstrategy}). After multiple adjustments, the light source finally reaches the target pose, i.e., ALR is finished. See the green curve of Fig.~{\ref{fig:balltransformation}} for an example of light source adjustment trajectory. Hence, the convergence of the ALR process relies on $\displaystyle \lim_{t \to \infty} {\rm{diag}}({\bm{\lambda}}_{t}){\bf{m}}_{t} = \mathbf{0}$.


\begin{lemma}[Convergence of ${\bm{\lambda}}_{t}$] \label{prop:ALR_convergence}
	Using the bisection approaching strategy, i.e., Eq.~(\ref{eq:bisectstrategy}), if bisection occurs infinite times, we have $\displaystyle \lim_{t \to \infty}{\bm{\lambda}}_{t} = \mathbf{0}$ if $\mu<2$. 
\end{lemma}
\begin{proof} 
	We have ${\bm{\lambda}}_{t}=[r_{t}^{\rm{\lambda}},\theta_{t}^{\rm{\lambda}},\varphi_{t}^{\rm{\lambda}}]^{\rm{T}}$ and we first prove the convergence of ${r}^{\lambda}_{t}$. As shown in Fig.~\ref{fig:ALRconvergence}(a), we assume the light source has been adjusted to ${\bm{\rho}}_{n}$ along the radial-axis from initial pose ${\bm{\rho}}_{0}$, and the light source should cross the reference pose ${\bm{\rho}}^{\rm ref}$ at next iteration $n+1$. Note, the solid and dashed lines indicate the adjustments that have and have not yet occurred, respectively. Then, we have the conclusion that if $\mu<2$, Fig.~\ref{fig:ALRconvergence}(b) (or (c)) and Fig.~\ref{fig:ALRconvergence}(d) (or (e)) happen alternately. Specifically, after $n$th iteration, i.e., Fig.~\ref{fig:ALRconvergence}(a), then Fig.~\ref{fig:ALRconvergence}(b) or (c) must happen and followed by Fig.~\ref{fig:ALRconvergence}(d) or (e), and Fig.~\ref{fig:ALRconvergence}(b) or (c) happens once again after that. The alternation process lasts forever. Note, here we only consider $\mu<2$. If $\mu\ge 2$, the above alternation process may not be satisfied. For Fig.~\ref{fig:ALRconvergence}(b) or (d), we have
	\begin{equation} \label{eq:ALR_convergence_1}
	\begin{aligned}
	{r}^{\lambda}_{m+1} = \frac{1}{2}{r}^{\lambda}_{m},
	\end{aligned}
	\end{equation}
	where $m\ge n$. Similarly, for Fig.~\ref{fig:ALRconvergence}(c) or (e), we have
	\begin{equation} \label{eq:ALR_convergence_2}
	\begin{aligned}
	{r}^{\lambda}_{m+2} = \mu{r}^{\lambda}_{m+1} = \frac{\mu}{2}{r}^{\lambda}_{m},
	\end{aligned}
	\end{equation}
	where $\mu$ is the speed-up rate of navigation magnitude. Therefore, for $t$th iteration ($t>n$), we have 
	\begin{equation} \label{eq:ALR_convergence_3}
	\begin{aligned}
	{r}^{\lambda}_{t} = \mu^{n-1} {\Big(\frac{1}{2}\Big)}^{i} {\Big(\frac{\mu}{2}\Big)}^{j}
	{r}^{\lambda}_{0},
	\end{aligned}
	\end{equation}
	where $i$ and $j$ denote the occurrence times of Eq.~(\ref{eq:ALR_convergence_1}) and Eq.~(\ref{eq:ALR_convergence_2}), respectively. In fact, $t = n+i+2j$. Since $n$ is finite, when $t$ approaches infinity, $i$ and $j$ also approach infinity. Hence, we have $\displaystyle \lim_{t \to \infty}{r}^{\lambda}_{t} = 0$ if $\mu<2$. The convergence proof of ${\theta}^{\lambda}_{t}$ or ${\varphi}^{\lambda}_{t}$ is same to ${r}^{\lambda}_{t}$ and we don't provide the details here. Then, we have $\displaystyle \lim_{t \to \infty}{\bm{\lambda}}_{t} = \mathbf{0}$ when $\mu<2$.
\end{proof}
According to Lemma~{\ref{prop:ALR_convergence}}, we have $\displaystyle \lim_{t \to \infty} {\rm{diag}}({\bm{\lambda}}_{t}){\bf{m}}_{t} = \mathbf{0}$. Fig.~\ref{fig:simulation_mu} demonstrates the convergence process of the navigation vector ${\rm{diag}}({\bm{\lambda}}_{t}){\bf{m}}_{t}$ under different speed-up rate $\mu$ by simulation experiment. We set the target light source pose, initial light source pose and initial navigation magnitude to $[60,-70,80]^{\rm T}$, $[0,0,0]^{\rm T}$ and $[5,5,5]^{\rm T}$, respectively. Then we conduct ALR under different speed-up rates. From Fig.~\ref{fig:simulation_mu}, we can find that, $\mu>1$ ($\mu=1.2$, $\mu=1.5$ in Fig.~\ref{fig:simulation_mu}) may effectively help the light source to quickly approach the target and accelerate the convergence compared with $\mu=1$. Besides, Fig.~\ref{fig:simulation_mu} also verifies that ALR does not converge if $\mu\ge 2$. In fact, the selection of $\mu$ is related to many factors, e.g., scene, initial light source pose, initial navigation magnitude of ALR. Empirically, we set $\mu$ to $1.2$ in our experiments.

\begin{figure}[t]
	\begin{center}
		\includegraphics[width=1.0\linewidth]{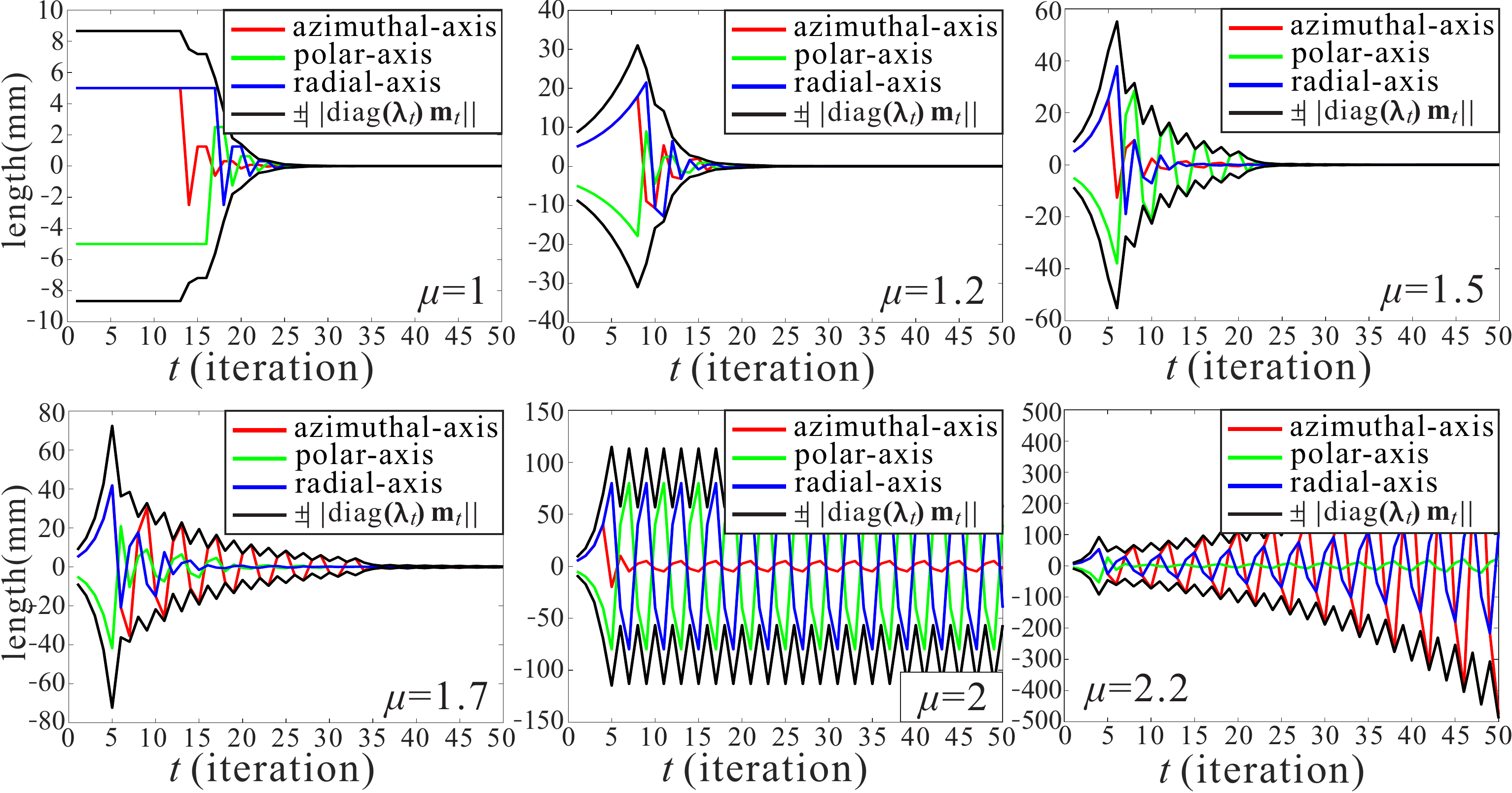} \vspace{-0.6cm}
	\end{center}
	\caption{ \small { Iterative changes of navigation vector ${\rm{diag}}({\bm{\lambda}}_{t}){\bf{m}}_{t}$ during ALR under different speed-up rate $\mu$ of navigation magnitude. See text for details. }}
	\label{fig:simulation_mu} \vspace{-0.2cm}
\end{figure}

\subsection {Invariance to $\mathbf{N}$ and $\mathbf{l}$ decomposition ambiguity} \label{sec:ambiguity}
According to Eq.~(\ref{eq:S=NL}), the $\mathbf{N}$ and $\mathbf{l}$ decomposition generally subjects to an ambiguity matrix $\bf{Z}$ satisfying ${\bf{S}} = {\bf{N}} {\bf{l}} = {\bf{\tilde N}} {\bf{Z}}^{-1} {\bf{Z}} {\bf{\tilde l}}$, where ${\bf{\tilde N}} \in \mathbb{R}^{P \times 3}$ and ${\bf{\tilde l}} \in \mathbb{R}^{3}$ are the real normal and lighting vector. Therefore, our initialization also exits an ambiguity matrix ${\bf{Z}}$ between the calculated and real lighting vectors (or scene normals), which may influence the correctness of ALR. Fortunately, we prove that if the angle difference between ${\bf{\tilde l}}$ and ${\bf{l}}$ is not larger than $\frac{\pi}{3}$, the ambiguity matrix $\bf{Z}$ does not affect the effectiveness and convergence of our ALR approach.

\begin{lemma} \label{lemma:ZisR}
	The ambiguity matrix $\bf{Z}$ generated by the decomposition of Eq.~(\ref{eq:S=NL}) is a rotation matrix.
\end{lemma}
\begin{proof} 
	Since ${\bf{\tilde N}} {\bf{Z}}^{-1} = {\bf{N}}$ is a full rank linear system, ${\bf{Z}}$ has a unique solution. Since both ${\bf{\tilde N}}_{p}$ and ${\bf{N}}_{p}$ are $1\times3$ row unit vectors for any point $p$, there is a rotation matrix $\bf{T}$ satisfying ${\bf{T}}{\bf{\tilde N}}_{p}^{\rm{T}}={\bf{N}}_{p}^{\rm{T}}$, then ${\bf{T}} = \bf{Z}$, i.e., ${\bf{Z}}$ is a rotation matrix.
\end{proof}

\begin{lemma} \label{lemma:z-axis}
	The ALR navigation guidance in radial-axis is independent to $\bf{Z}$.
\end{lemma}
\begin{proof} From Proposition~{\ref{prop:ALR_apl}}, we know that for any point $p$ in SIC $\bf C$, $\bf{l}$ and ${\bf{N}}^{\rm{s}}_{p}$ have the same included angle $\alpha$. We have
	\begin{equation} \label{eq:z-axis_1}
	\begin{aligned}
	{{\bf{B}}_p} = {\bf{N}}_p^{\rm{s}} {{\bf{l}}}= {\left\| {\bf{l}} \right\|} {\rm{ cos\ }} {\alpha}.
	\end{aligned}
	\end{equation}
	Let $r$ be the radius of $\bf{C}$, then the area $A$ of $\bf{C}$ satisfies
	\begin{equation} \label{eq:z-axis_2}
	\begin{aligned}
	A=\pi r^2 = \pi {\rm{sin}}^{2}(\alpha) = \pi (1-{\rm{cos}^{2}(\alpha)}).
	\end{aligned}
	\end{equation}
	Since $\alpha \in [0, \frac{\pi}{2}]$, combining Eq.~(\ref{eq:z-axis_1}) and Eq.~(\ref{eq:z-axis_2}), we have ${{\bf{B}}_p} = {\left\| {\bf{l}} \right\|} \sqrt{1-\frac{A}{\pi}}$. Since the intensities of all points in reference SIC ${\bf{C}}^{\rm{ref}}$ and current SIC $\bf{C}$ are the same, so ${\left\| {\bf{l}}^{\rm{ref}} \right\|} \sqrt{1-\frac{{A}^{\rm{ref}}}{\pi}} = {\left\| {\bf{l}} \right\|} \sqrt{1-\frac{A}{\pi}}$. We can see ${A}^{\rm{ref}} \star A \Leftrightarrow \left\| {\bf{l}}^{\rm{ref}} \right\| \star \left\| {\bf{l}} \right\|$, where $\star$ denotes the operator in $\{<, =, >\}$, i.e., $A$ and $\left\| \bf{l} \right\|$ have the same magnitude relation. Besides, we know $\left\| {\bf{l}}^{\rm{ref}} \right\| = \left\| {{\bf{Z}} {\bf{\tilde l}}^{\rm{ref}} } \right\| = \left\| { {\bf{\tilde l}}^{\rm{ref}} } \right\|$ and $\left\| {\bf{l}} \right\| = \left\| {{\bf{Z}} {\bf{\tilde l}} } \right\| = \left\| { {\bf{\tilde l}} } \right\|$. It means $A^{\rm{ref}}$ and $A$ have the same magnitude relation to ground truth ${\tilde A}^{\rm{ref}}$ and $\tilde A$. Thus, the ambiguity matrix $\bf{Z}$ does not influence the ALR navigation in radial direction.
\end{proof}

\begin{lemma} \label{lemma:xy-axis}
	Let $\langle \beta, {\bf{e}} \rangle$ be the axis-angle representation of ${\bf{Z}}$. For the azimuthal and polar direction in Eq.~(\ref{eq:navcom}), if $\beta\leq\frac{\pi}{3}$, ${\bf{Z}}$ does not affect our ALR process. We can faithfully relocalize the light source to reference apl pose.
\end{lemma}
\begin{proof} The azimuthal and polar adjustments of light source on sphere only cause the change of lighting direction. Since we use the same $\bf R$ and $\bf N$ to calculate ${\bf{l}}_t$ and ${\bf{l}}^{\rm ref}$, it satisfies ${\bf{l}}^{\rm{ref}}={\bf{Z}} {\bf{\tilde l}}^{\rm{ref}}$ and ${\bf{l}}_{t} =  {\bf{Z}} {\bf{\tilde l}}_{t}$, where ${\bf{\tilde l}}^{\rm ref}$ and ${\bf{\tilde l}}_{t}$ are the real reference and current lighting vectors. From Fig.~{\ref{fig:threemodel}}(d), we have
	\begin{equation} \label{eq:xy-axis_1}
	\begin{aligned}
	{\bf{l}}^{\rm{ref}} = {{\bf{T}}}_{t} {\bf{l}}_{t} = {{\bf{T}}}_{t+1} {\bf{l}}_{t+1}, \quad {\bf{l}}_{t} =  {\bf{Z}} {\bf{\tilde l}}_{t}, \quad  {\bf{\tilde l}}_{t+1} = {\bf{P}}_{t} {\bf{\tilde l}}_{t},
	\end{aligned}
	\end{equation}
	where $t$ denotes $t$th iteration, rotation matrix ${\bf{T}}_{t}$ denotes the matrix form of navigation from current lighting to the reference one, ${\bf{P}}_{t}$ is the real light source adjustment realized by robotic platform or hand. From Eq.~(\ref{eq:xy-axis_1}), we have ${\bf{T}}_{t+1} = {\bf{T}}_{t} {\bf{Z}} {\bf{P}}_{t}^{-1} {\bf{Z}}^{-1}$. Generally, we pursue to adjust light source by ${\bf{T}}_{t}$, i.e., ${\bf{P}}_{t} = {\bf{T}}_{t}$. So, ${\bf{T}}_{t+1} = {\bf{T}}_{t} {\bf{Z}} {\bf{T}}_{t}^{-1} {\bf{Z}}^{-1}$. From Theorem 1 of~\cite{our2018ACR}, we know ${\bf{T}}_{t}$ infinitely approaches unit matrix if $\beta \leq \frac{\pi}{3}$, i.e., ${\bf{l}}_{t}$ and ${\bf{l}}^{\rm{ref}}$ are coincided. Since ${\bf{\tilde l}}^{\rm{ref}} = {\bf{Z}}^{-1}{\bf{l}}^{\rm{ref}}$ and ${\bf{\tilde l}}_{t} = {\bf{Z}}^{-1}{\bf{l}}_{t}$, ${\bf{\tilde l}}_{t}$ and ${\bf{\tilde l}}^{\rm{ref}}$ are coincided too. Hence, we can say that users can always relocalize the light source to the reference pose if $\beta \leq \frac{\pi}{3}$.
\end{proof}

The convergence condition of $\bf Z$ in Lemma~{\ref{lemma:xy-axis}} is equivalent to the one of lighting vector, i.e., if the angle difference between ${\bf{\tilde l}}$ and ${\bf{l}}$ is not larger than $\frac{\pi}{3}$, $\bf{Z}$ does not affect the convergence of our ALR-apl. In fact, the convergence condition can be easily satisfied by current photometric stereo methods~\cite{PS:ACFSUPSDM12}, which is empirically verified and discussed in detail in Sec.~\ref{exp:covgvalid}.

\section{ALR under More Realistic Lighting}\label{sec:ALR-NPLNSL}
In fact, ideal parallel lighting is inexistent in the real world. The commonly-used realistic light sources include near point light (NPL) source and small near surface light (sNSL) source. However, the NPL model or sNSL model usually cannot satisfy the requirement of real-time navigation response of ALR, because of the model's sophistication. Fortunately, in this paper, we prove that using the proposed ALR-apl method can also be applied to the above two kinds of light sources.

\subsection{ALR under near point lighting} \label{sec:NPL}
From Fig.~\ref{fig:threemodel}(b), near point lighting assigns different scene points with distance-related lighting directions,
\begin{equation} \label{eq:NPLmodel}
\begin{aligned}
{\bf{I}}_p=&{\rm{F}}_{\rm{NPL}}( {\rm{\mathcal L}}({\rm{\bm{\rho}}} \mid {\bf{\Theta}}), {\mathcal{S}}_{\tau + \vartriangle \tau}(p)) \\=&{{\bf{R}}_{p}} {{\bf{S}}_{p}} = {e} {\bf{N}}_p  {{({{\bm{\rho}}} - {{\bf{X}}_p})}} /{{{{\left\| {{{\bm{\rho}}} - {{\bf{X}}_p}} \right\|}^3}}}{{\bf{R}}_{p}},
\end{aligned}
\end{equation}
where ${{\bf{S}}} \in \mathbb{R}^{P}$ is the shading image, $\bm{\rho} \in \mathbb{R}^{3}$ is the near point light source position, ${\bf{X}}_p \in \mathbb{R}^{3}$ indicates the spatial coordinate of point $p$, $e$ is the lighting power. 

\begin{prop}[SICs \& shading equivalence] \label{prop:ALR_npl}
	Under NPL model, Given an arbitrary near point light source position, the spherical isointensity set $\bf{C}$ acquired from the rendered image $\bf{B}$ always forms a circle, under the view that light source points to the sphere center. Iff the reference and current SICs ${\bf{C}}^{\rm{ref}}$ and ${\bf{C}}_{t}$ coincide completely, the reference and current images, ${\bf{I}}^{\rm{ref}}$ and ${\bf{I}}_{t}$, are the same.
\end{prop}
\begin{proof} SICs: We replace $\bf{N}$ by a sphere normal $\bf{N}^{\rm{s}}$. Let $\bf E$ be the light source position. From Fig.~{\ref{fig:threemodel}}(b), we have
	\begin{equation} \label{eq:NPLM_2}
	\begin{aligned}
	{{\bf{B}}_{p}} = {\bf{N}}_p^{\rm{s}} {{\bf{l}}_p} \frac{e}{\left\| {{\bf{l}}_p} \right\|^3}
	= {\rm{  cos\ }}\gamma_{p}  \frac{e}{\left\| {{\bf{l}}_p} \right\|^2},
	\end{aligned}
	\end{equation}
	where ${\bf{l}}_p \in \mathbb{R}^{3}$ indicates the lighting vector and satisfies ${\bf{l}}_p={\bf{E}}-{{\bf{X}}_{p}}$, $\left\| {{\bf{l}}_p} \right\|$ represents the distance between $\bf{E}$ and ${\bf{X}}_p$, $\gamma _{p}$ denotes the included angle between ${\bf{l}}_p$ and ${\bf{N}}_p^{\rm{s}}$. Since ${\rm{cos\ }}\gamma _{p}$ and $\left\| {\bf{l}}_p \right\|$ have the same magnitude relation, thus
	\begin{equation}
	\begin{aligned}
	{{\bf{B}}_p} = {{\bf{B}}_q} 
	\Leftrightarrow ~\gamma _{p}=\gamma  _{q}, {\ }
	\left\| {{\bf{l}}_p} \right\| = \left\| {{\bf{l}}_q} \right\|.
	\end{aligned}
	\end{equation}
	So we have the same conclusion like Proposition~\ref{prop:ALR_apl}, i.e., the spherical isointensity set $\bf{C}$ forms a circle under the view that the point light source towards the sphere center.
	
	Shading equivalence: We first prove the sufficient condition. For each pixel $p$ in $\bf{C}^{\rm{ref}}$ (and $\bf{C}$) on $\bf{B}^{\rm{ref}}$ (and $\bf{B}$), we have
	${{\bf{B}}^{\rm{ref}}_{p}}={\rm{  cos\ }}\gamma^{\rm{ref}}_{p} \frac{e}{\left\| {{\bf{l}}^{\rm{ref}}_p} \right\|^2}$ and ${{\bf{B}}_{p}}= {\rm{ cos\ }}{\gamma} _{p} \frac{e}{\left\| {{\bf{l}}_p} \right\|^2}$.
	Since ${\bf{C}}^{\rm{ref}} = \bf{C}$, ${\bf{B}}^{\rm{ref}}_{p}={\bf{B}}_{p}$.
	Refer to Proposition~{\ref{prop:ALR_apl}}, $ {\rm{  cos\ }} \gamma _{p}$ and $\left\| {\bf{l}}_p \right\|$ have the same magnitude
	relation, we have $\gamma^{\rm{ref}} _{p} = {\gamma} _{p}$, then
	\begin{equation} \label{eq:NPLM_3}
	\begin{aligned}
	\gamma^{\rm{ref}}_{p} = {\gamma} _{p} \Leftrightarrow ~& {\bf{N}}_p^{\rm{s}} {{\bf{l}}^{\rm{ref}}_p} = {\bf{N}}_p^{\rm{s}} {{\bf{l}}_p} \\
	\Leftrightarrow ~& {\bf{N}}_p^{\rm{s}} ({\bf{E}}^{\rm{ref}} - {{\bf{X}}_p}) = {\bf{N}}_p^{\rm{s}} ( {\bf{E}} - {{\bf{X}}_p}) \\
	\Leftrightarrow ~& {\bf{N}}_p^{\rm{s}} {\bf{E}}^{\rm{ref}} = {\bf{N}}_p^{\rm{s}} {{\bf{E}}}.
	\end{aligned}
	\end{equation}
	Since Eq.~(\ref{eq:NPLM_3}) holds for any point in ${\bf{C}}^{\rm{ref}}$ and $\bf{C}$, we have ${\bf{E}}^{\rm{ref}} = \bf{E}$, i.e., the reference and current positions of near point light source are the same, so ${\bf{I}}^{\rm{ref}} = \bf{I}$. The proof of necessary condition is similar to the sufficient condition and we do not provide the details.
\end{proof}

\begin{prop}[ALR navigation equivalence] \label{prop:ALR_npl_equal}
	The ALR approach for the apl model is also applicable to NPL model.
\end{prop}
\begin{proof}
	As shown in Fig.~\ref{fig:threemodel}(b), there is an approximate parallel lighting $\bf{l}$ which parallels to the direction that the NPL towards sphere center. We use superscripts `apl' and `npl' to distinguish the two models. Given any two points $p$ and $q$, it satisfies
	\begin{equation} \label{eq:NPLM_4}
	\begin{aligned}
	&{\bf{B}}^{\rm{apl}}_p \star {\bf{B}}^{\rm{apl}}_q
	\Leftrightarrow ~\alpha_{p} \star \alpha_{q} \\
	\Leftrightarrow ~&\gamma _{p} \star \gamma _{q},~ \left\| {{\bf{l}}_p} \right\| = \left\| {{\bf{l}}_q} \right\| 
	\Leftrightarrow ~{\bf{B}}_p^{\rm{npl}} \star {\bf{B}}_q^{\rm{npl}},
	\end{aligned}
	\end{equation}
	where $\star$ denotes an operator in $\{ \rm{>,<,=} \}$. That is, ${\bf{B}}_p^{\rm{apl}}$ and ${\bf{B}}_p^{\rm{npl}}$ have the same magnitude relation. Thus, combining Proposition~{\ref{prop:ALR_npl}}, the proposed ALR-apl strategy can also be applied to near point lighting condition.
\end{proof}

Proposition~\ref{prop:ALR_npl} is similar to Proposition~\ref{prop:ALR_apl}. Besides, Proposition~\ref{prop:ALR_npl_equal} guarantees we can directly apply ALR-apl to NPL and avoid solving the complex inverse problem for NPL model.

\subsection{ALR under small near surface lighting}
\begin{figure}[t]
	\begin{center}
		\includegraphics[width=1.0\linewidth]{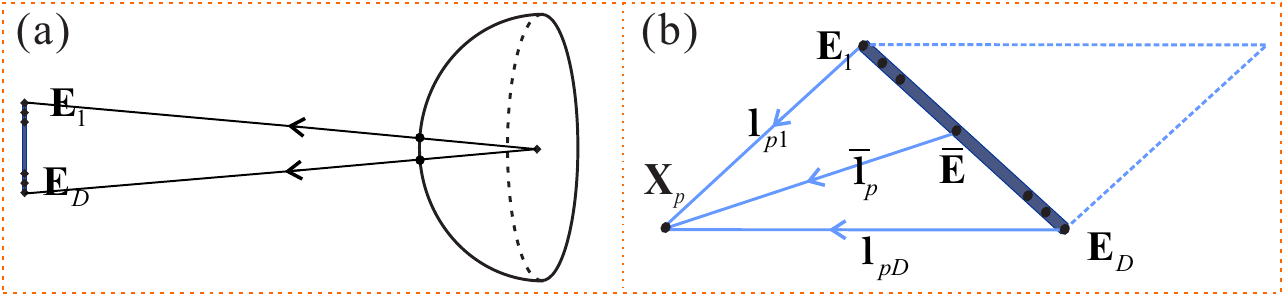}\vspace{-0.3cm}
	\end{center}
	\caption{\small { (a) The size of sNSL is much less than the distance between sNSL and scene center; (b) The sNSL can be described as a combination of multiple NPLs.}}
	\label{fig:distantNSL} \vspace{-0cm}
\end{figure}

Refer to Fig.~\ref{fig:threemodel}(c), a small near surface light source (sNSL) can be seen as the combination of many near point light sources. Thus, under sNSL model, we have 
\begin{equation} \label{eq:sNSLmodel}
\begin{aligned}
{\bf I}_p = &{\rm{F}}_{\rm{sNSL}}( {\rm{\mathcal L}}({\rm{\bm{\rho}}} \mid {\bf{\Theta}}), {\mathcal{S}}_{\tau + \vartriangle \tau}(p)) \\=& {{\bf{R}}_{p}}{{\bf{S}}_{p}} = \sum_{d=1}^D e{\bf{N}}_p  \frac{{({{\bf{E}}}_{d} - {{\bf{X}}_p})}}{{\left\| {{{\bf{E}}}_{d} - {{\bf{X}}_p}} \right\|}^3}{{\bf{R}}_p},
\end{aligned}
\end{equation}
where ${\bf{E}}_{d}$ indicates the $d$th near point light source, $e$ denotes the lighting power of all NPLs, $D$ is the number of NPLs, ${\bf{X}}_p \in \mathbb{R}^{3}$ indicates the spatial coordinate of point $p$. 
We can find that solving $\bm \rho$ from Eq.~(\ref{eq:sNSLmodel}) is a nonlinear problem which is hard to be solved. Replacing the scene normal by the sphere normal ${\bf{N}}^{\rm{s}}$, we have
\begin{equation} \label{eq:sNSL_1}
\begin{aligned}
{{\bf{B}}_{p}}= \sum_{d=1}^D {{\bf{N}}_p^{\rm{s}}}  {{{\bf{l}}_{pd}}}  \frac{e}{\left\| {{\bf{l}}_{pd}} \right\|^3},
\end{aligned}
\end{equation}
where ${\bf{l}}_{pd}$ indicates the lighting vector that $p$th scene point towards $d$th NPL, and it satisfies ${\bf{l}}_{pd} = {\bf{E}}_{d} - {\bf{X}}_{p}$. As shown in Fig.~\ref{fig:distantNSL}(a), generally, the size of sNSL is much less than the distance between sNSL and the scene center. Besides, we take no account of the rotation of sNSL and keep the sNSL midperpendicular always pointing to the scene center during ALR process. Under these constraints, we can approximatively consider that each NPL in sNSL has the same distance to a certain scene point ${\bf{X}}_{p}$, i.e., $\left\|{\bf{l}}_{p1}\right\|=\left\|{\bf{l}}_{p2}\right\|=...
=\left\|{\bf{l}}_{pD}\right\|=\left\|{\bf{\bar l}}_{p}\right\|$, where ${\bf{\bar l}}_{p}$ is the lighting vector that sNSL center points to scene point ${\bf{X}}_{p}$. Then we rewrite Eq.~(\ref{eq:sNSL_1}) as
\begin{equation} \label{eq:sNSL_2}
\begin{aligned}
{{\bf{B}}_{p}} =  {{\bf{N}}_p^{\rm{s}}}  \frac{e}{\left\| {{\bf{\bar l}}_{p}} \right\|^3} \sum_{d=1}^D {{{\bf{l}}_{pd}}}.
\end{aligned}
\end{equation}
Refer to Fig.~{\ref{fig:distantNSL}}(b), we have
\begin{equation} \label{eq:sNSL_3}
\begin{aligned}
&{\bf{l}}_{p1} + {\bf{l}}_{pD} = 2 {\bf{\bar l}}_{p}, \quad {\bf{l}}_{p2} + {\bf{l}}_{pD-1} = 2 {\bf{\bar l}}_{p}, \quad ......
\end{aligned}
\end{equation}
where ${\bf{\bar l}}_{p}={\bf{\bar E}}-{\bf{X}}_{p}$, and ${\bf{\bar E}}$ indicates the coordinate of sNSL center. According to Eq.~(\ref{eq:sNSL_3}), we can get $\sum_{d=1}^D {{\bf{l}}_{pd}} = D  {\bf{\bar l}}_{p}$. Then we rewrite Eq.~(\ref{eq:sNSL_2}) as
\begin{equation} \label{eq:sNSL_4}
\begin{aligned}
{{\bf{B}}_{p}}= D {{\bf{N}}_p^{\rm{s}}} {{{\bf{\bar l}}_{p}}}  \frac{e}{\left\| {{\bf{\bar l}}_{p}} \right\|^3}.
\end{aligned}
\end{equation}
We can see that $D$ is a constant and Eq.~(\ref{eq:sNSL_4}) has the same form with Eq.~(\ref{eq:NPLM_2}), so the sNPL can be seen as a NPL, and the spatial coordinate is $\bf{\bar E}$. Similar to the propositions about NPL, we can take the conclusion that the ALR-apl method can be applied to sNSL condition directly.

\section{Experimental Results}

\begin{figure}[!htb]
	\begin{center}
		\includegraphics[width=0.8\linewidth]{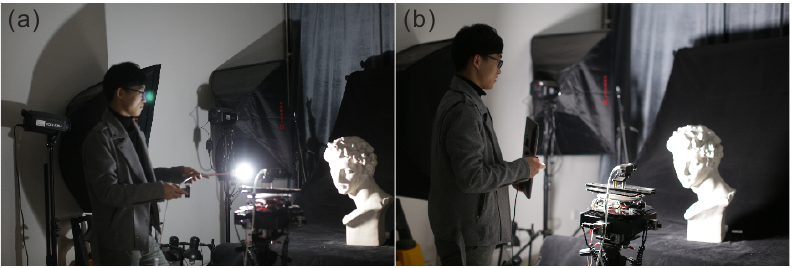} \vspace{-0.4cm}
	\end{center}
	\caption{\small { Working scene with near point light source (a) and small
			near surface light source (b).}}
	\label{fig:workingscene}\vspace{-0cm}
\end{figure}

\subsection{Setup}
We build 13 different scenes (S1--13) for evaluating the proposed ALR method. S1--3, S11 and S13 mainly exhibit near-Lambertian surface. The other scenes involve some non-Lambertian regions, e.g., transparency (S4, S12), hard cast shadow (S5--6) and specularity (S7--10, S12). See Fig.~\ref{fig:workingscene}, we use a small lamp bulb and a small handheld LED surface light source as the near point light (NPL) source and small near surface light source (sNSL), respectively. Besides, we employ a consumer robotic arm (uArm Swift Pro) to verify the effectiveness of the automatic ALR. All images are captured by a Canon 5D Mark III camera. Experiments are conducted in a consumer computer with i7 CPU.

\subsection{Convergence and effectiveness validation} \label{exp:covgvalid}


Refer to Eq.~(\ref{eq:bisectstrategy}), we use a \emph{bisection approaching strategy} to calculate the navigation vector ${\rm{diag}}({\bm{\lambda}}_t){\bf{m}}_t$ for $t$th iteration of ALR.
To verify the convergence of navigation vector, we conduct ALR using the robotic arm for scenes S4 and S5, then we record the navigation vector of each iteration. Fig.~\ref{fig:stepconvergence} shows the relation of the navigation vector (absolute value) and the iteration number for the two scenes. We can clearly see that the navigation direction is alternately changed during ALR process, and the navigation magnitude increases in the beginning and then gradually converges to zero. See Sec.~\ref{sec:ALRconvergence} for the convergence proof of navigation vector.


\begin{figure}[t]
	\begin{center}
		\includegraphics[width=1.0\linewidth]{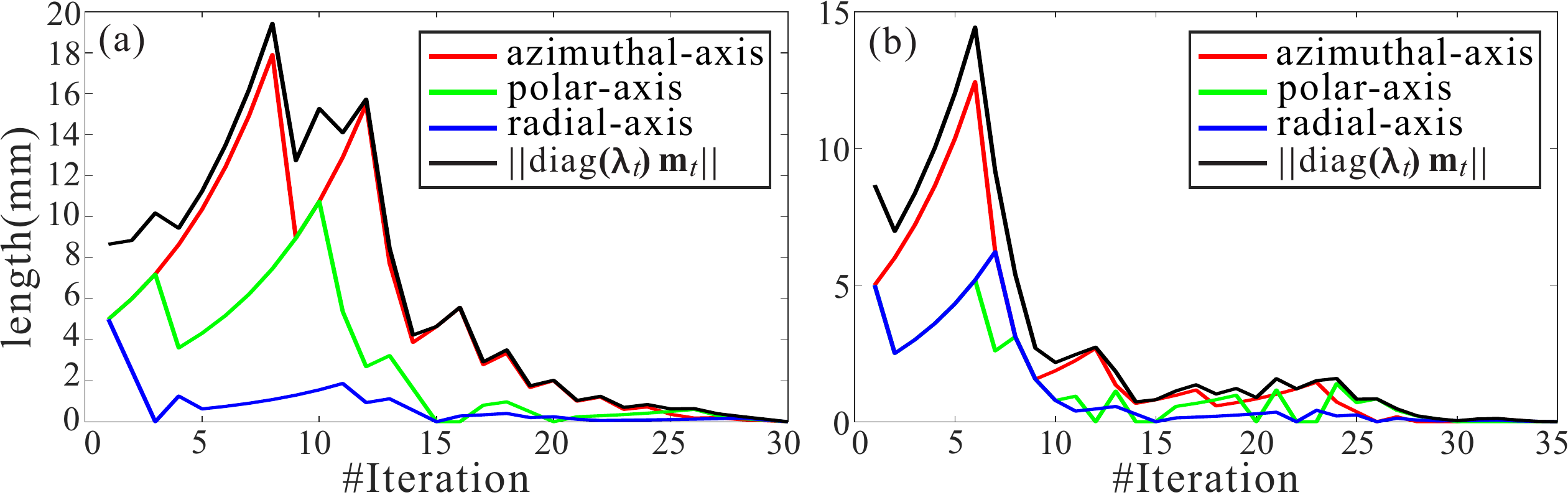} \vspace{-0.6cm}
	\end{center}
	\caption{ \small { Relation of the navigation vector and iteration number during ALR using robotic arm for scenes S4 (a) and S5 (b). }}
	\label{fig:stepconvergence} \vspace{-0.2cm}
\end{figure}

We use the ALR goodness $g$, i.e., the Intersection-over-Union (IoU) as the termination condition for ALR process. To verify the effectiveness, we conduct ALR using the robotic arm for S1 and S7, and we record the image and the corresponding $g$ for each iteration. Then we calculate the SSIM and MS-SSIM~\cite{wang2003multiscale} scores for each image. Fig.~\ref{fig:IoU}(a)--(b) show the relation of ALR goodness $g$ and the (MS-)SSIM score. Besides, the 10-times magnified absolute differences between the reference image and some captured images during ALR are also shown. Since the slight change of lighting condition may cause large difference of image appearance, it is reasonable that the (MS-)SSIM decreases slightly sometimes (see the variation tendency of (MS-)SSIM in Fig.~\ref{fig:IoU}(b)). Nevertheless, the (MS-)SSIM increases with the increase of $g$ on the whole. 

\begin{figure}[t]
	\begin{center}
		\includegraphics[width=1\linewidth]{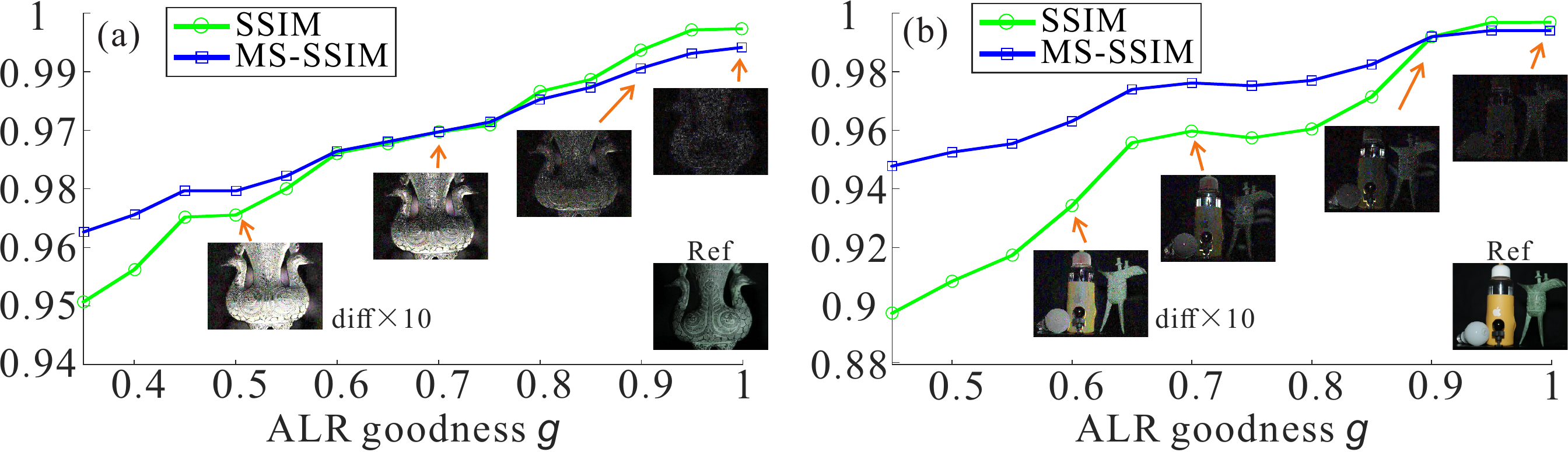} \vspace{-0.6cm}
	\end{center}
	\caption{ \small { Relation of ALR goodness $g$ and (MS-)SSIM for scenes S1 (a) and S7 (b). The difference images of reference image and some captured images during ALR are shown.}}
	\label{fig:IoU} \vspace{-0.0cm}
\end{figure}

Refer to Sec.~\ref{sec:ambiguity}, we prove that if the angle difference between ${\bf{\tilde l}}$ and ${\bf{l}}$ is not larger than $\frac{\pi}{3}$, the ambiguity matrix $\bf{Z}$ does not affect the convergence of our ALR approach. In fact, the condition is easy to be met. To verify this, we introduce the dataset~\cite{PS:FSTLS2015} which includes 7 statue scenes (e.g., Cat, Frog, Hippo) and corresponding ground truth normals. Each scene has 20 multi-illumination images. We calculate the scene normal by LDR~\cite{PS:ACFSUPSDM12} and calculate the mean angle error (MAE) for each scene. The MAE scores are shown in Table~\ref{table:angleLessThannpi3}. We find that these MAEs are much less than $\frac{\pi}{3}$. Therefore, according to Lemma~{\ref{lemma:xy-axis}}, we can confidently say that the ambiguity matrix $\bf{Z}$ does not influence our ALR method.

\begin{table}[!htb] \vspace{-0cm}
	\caption{\small Mean angle error (MAE) of the calculated normal.} \vspace{-0.5cm} \label{table:angleLessThannpi3}
	\begin{center}
		\renewcommand\arraystretch{1.4}
		\begin{tabular}{|c||c|c|c|c|c|c|c|} \hline
			Scene & Cat & Frog  & Hippo & Lizard & Pig & Scholar & Turtle\\
			\hline &&&&&&&\vspace{-0.425cm} \\ 
			MAE  & $\frac{\pi}{34}$ & $\frac{\pi}{28}$ & $\frac{\pi}{29}$ & $\frac{\pi}{45}$ & $\frac{\pi}{27}$ & $\frac{\pi}{15}$ & $\frac{\pi}{31}$\\
			\hline
		\end{tabular}
	\end{center}
\end{table}

The proposed ALR method only need a small amount of in-situ captured images for initialization. To verify this, we conduct ALR using robotic arm for S1--5 with different numbers of in-situ images. The 2nd to 5th rows of Table~\ref{table:diffimagenumber} show the averages of 4 commonly-used image similarity criteria, i.e., MSE, PSNR, SSIM and MS-SSIM~\cite{wang2003multiscale} under different image numbers. The 6th row of Table~\ref{table:diffimagenumber} shows the mean angle error (MAE) between the calculated scene normal and the one using 100 images. We can see the MAE reduces with the increase of image number, and the recurrence accuracy of ALR is stable and always well. In fact, our ALR method is effective as long as the calculated scene normal satisfies Lemma~{\ref{lemma:xy-axis}}, and more images cannot help to improve the recurrence accuracy. Empirically, we use 13 multi-illumination images for calculating scene normal and reflectance.

\begin{table}[!htb] \vspace{-0cm}
	\caption{\small {ALR accuracy vs. \#images used in initialization.}}\label{table:diffimagenumber}  \vspace{-0.4cm}
	\begin{center}
		\scriptsize
		\begin{tabular}{@{}|@{\hspace{2pt}}c@{\hspace{2pt}}||@{\hspace{2pt}}c@{\hspace{2pt}}|@{\hspace{2pt}}c@{\hspace{2pt}}|@{\hspace{2pt}}c@{\hspace{2pt}}|@{\hspace{2pt}}c@{\hspace{2pt}}|@{\hspace{2pt}}c@{\hspace{2pt}}|@{\hspace{2pt}}c@{\hspace{2pt}}|@{\hspace{2pt}}c@{\hspace{2pt}}|} \hline
			\# Imgs  & 5 & 10  & 20 & 40 & 60 & 80 & 100 \\
			\hline \hline
			MSE & 3.4289&	3.2433&	2.5803&	2.8792&	2.6086&	2.7137&	2.8708\\
			\hline
			PSNR &42.7792&	43.0208&	44.0140& 43.5379&	43.9665&	43.7950&	43.5506\\
			\hline
			MSSSIM & 0.9953&	0.9953&	0.9961&	0.9957&	0.9961&0.9960&	0.9957\\
			\hline
			MS\_MSSIM & 0.9964&	0.9964&	0.9967&	0.9965&	0.9967&	0.9966& 0.9966\\
			\hline \hline
			MAE & $\frac{\pi}{20}$ & $\frac{\pi}{23}$ & $\frac{\pi}{49}$ & $\frac{\pi}{87}$ & $\frac{\pi}{136}$ & $\frac{\pi}{305}$& 0\\
			\hline
		\end{tabular}
	\end{center}
\end{table}

\subsection{Quantitative comparison}
To compare with our ALR method, we use the small lamp bulb (NPL) and the small handheld LED surface light source (sNSL) to collect 13 multi-illumination images as the reference images for scenes S1--10 and scenes S11--13, respectively. We recur all the 13 side lightings using our ALR method for all scenes by hand (ALR\_H). Considering the robotic arm only has 3 degrees of freedom and it cannot keep the sNSL always towards the scene center, so we only conduct ALR with robotic arm (ALR\_R) for S11--13.
We use PTM~\cite{RTI:PTM01}, HSH~\cite{RTI:TAERIICLOUUNL11}, LDR~\cite{PS:ACFSUPSDM12} as the baselines. PTM~\cite{RTI:PTM01} and HSH~\cite{RTI:TAERIICLOUUNL11} are two image-based relighting methods. For PTM and HSH, we use a light probe to calibrate the lighting direction for each captured image. LDR~\cite{PS:ACFSUPSDM12} is a state-of-the-art uncalibrated photometric stereo method based on parallel lighting model. We carry out the 3 methods using the captured 13 images and generate 13 lighting recurrence images. We use MSE, PSNR, SSIM, MS-SSIM~\cite{wang2003multiscale} as accuracy metrics.

Fig.~\ref{fig:meanNPL} and Fig.~\ref{fig:meansNSL} show the quantitative comparisons of our ALR method and 3 baselines for scenes S1--10 and S11--13, respectively. Each node indicates the average evaluation of 13 lighting recurrence results and the up and down bar of each node denotes the variance. The average of each method for all scenes are also shown in Fig.~\ref{fig:meanNPL} and Fig.~\ref{fig:meansNSL}. Besides, all the scores of the 4 criteria for each scene can be approached in Tables~\ref{table:MSEvalue}--\ref{table:MSSSIMvalue}. We can see that our method ALR\_R or ALR\_H can achieve the best recurrence accuracy than baselines. Besides, from Tables~\ref{table:MSEvalue}--\ref{table:MSSSIMvalue}, we can find the evaluation scores of ALR\_R are slightly better than ALR\_H except for S8. This is because the light source adjustment by robotic arm is more stable than hand and it is more easy to get more accuracy recurrence results for robotic arm.

\begin{figure}[!htb]
	\begin{center}
		\includegraphics[width=1\linewidth]{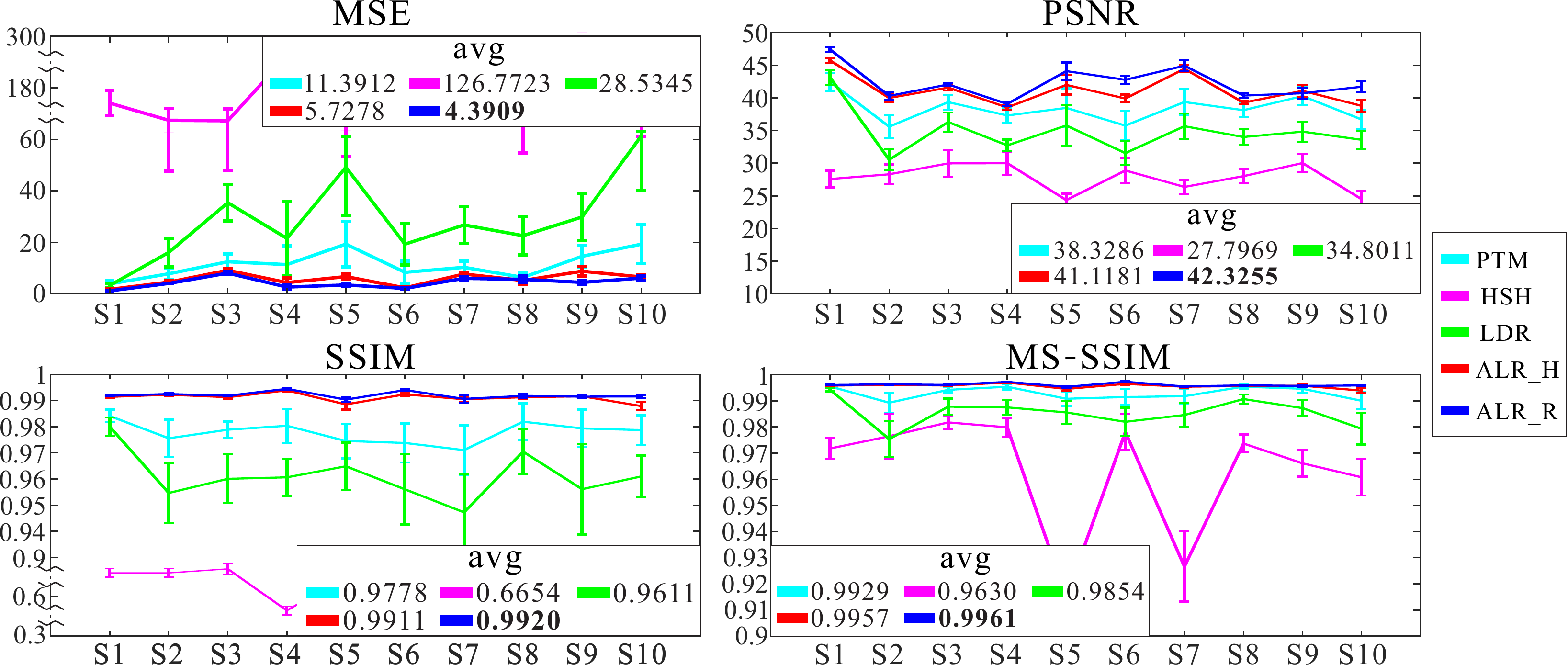} \vspace{-0.6cm}
	\end{center}
	\caption{\small {Quantitative comparison of near point light source.}} \vspace{-0cm}
	\label{fig:meanNPL} \vspace{-0cm}
\end{figure}

\begin{figure}[!htb]
	\begin{center}
		\includegraphics[width=1\linewidth]{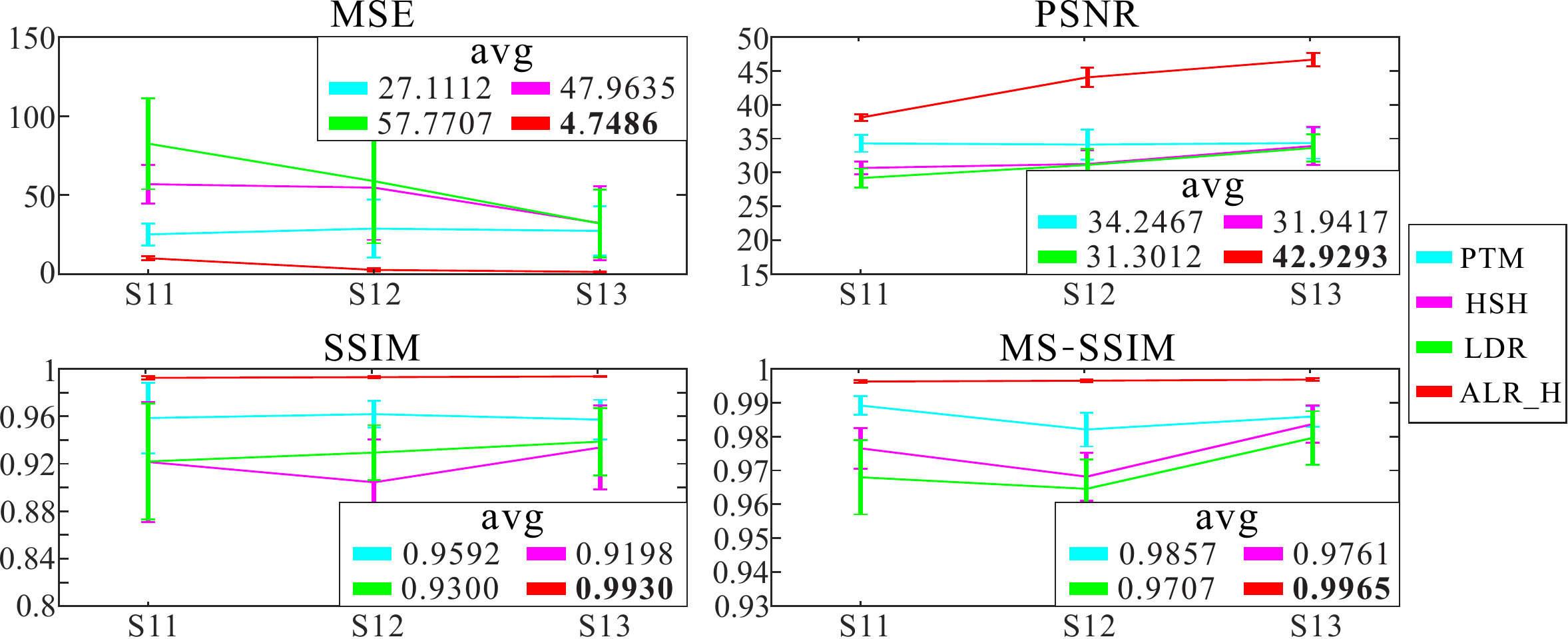} \vspace{-0.6cm}
	\end{center}
	\caption{\small {Quantitative comparison of small near surface light source.}} \vspace{-0cm}
	\label{fig:meansNSL} \vspace{-0cm}
\end{figure}


\begin{table}\vspace{-0.2cm}
	\fontsize{6pt}{6.75pt}\selectfont
	\caption{\small {Average MSEs for 13 scenes.}} \vspace{-0.5cm} \label{table:MSEvalue}
	\begin{center}
		\begin{tabular}{@{}|@{\hspace{1pt}}c@{\hspace{1pt}}||@{\hspace{1pt}}c@{\hspace{1pt}}|@{\hspace{1pt}}c@{\hspace{1pt}}|@{\hspace{1pt}}c@{\hspace{1pt}}|@{\hspace{1pt}}c@{\hspace{1pt}}|@{\hspace{1pt}}c@{\hspace{1pt}}|@{\hspace{1pt}}c@{\hspace{1pt}}|@{\hspace{1pt}}c@{\hspace{1pt}}|@{\hspace{1pt}}c@{\hspace{1pt}}|@{\hspace{1pt}}c@{\hspace{1pt}}|@{\hspace{1pt}}c@{\hspace{1pt}}|@{\hspace{1pt}}c@{\hspace{1pt}}|@{\hspace{1pt}}c@{\hspace{1pt}}|@{\hspace{1pt}}c@{\hspace{1pt}}|}
			\hline
			Method & S1 & S2 & S3 & S4 & S5 & S6 & S7 & S8 & S9 & S10 & S11 & S12 & S13  \\
			\hline		&&&&&&&&&&&&&\vspace{-0.18cm} \\ 
			\hline
			PTM & 3.89       & 7.83   & 12.49   & 11.39  &  19.33   &  8.39   & 10.32    & 6.41  &  14.60 & 19.28 & 25.16  & 28.82  & 27.36\\
			HSH & 118.16  &  72.12 &   70.67 &  244.34 &   92.67 &   154.44&   105.54 &   68.31  & 239.28 &  102.20  & 56.88 &  54.700  & 32.31\\
			LDR & 3.27  &   16.10 &   35.41  &  21.57  &  49.18 &   19.31  &  26.77  &  22.61   & 29.86  &  61.27 & 82.50  & 58.79 &  32.02 \\
			ALR\_H &1.75   &   4.57 &    9.16  &   4.39  &   6.72 &    2.36  &   7.72 &    {\bf 5.21}  &  8.81  &   6.59  & {\bf 10.10} &   {\bf 2.70} &   {\bf 1.44}\\
			ALR\_R &{\bf 1.19}  &   {\bf 4.07} &    {\bf 8.08}  &  {\bf 2.66} &    {\bf 3.48} &    {\bf 2.14}  &   {\bf 6.04} &  5.64  &  {\bf 4.48}  &   {\bf 6.13}  & NA& NA& NA\\
			\hline
		\end{tabular}
	\end{center}
\end{table}

\begin{table}\vspace{-0.2cm}
	\fontsize{6pt}{6.75pt}\selectfont
	\caption{\small Average PSNRs for 13 scenes.} \vspace{-0.5cm} \label{table:PSNRvalue}
	\begin{center}
		\begin{tabular}{@{}|@{\hspace{1.8pt}}c@{\hspace{1.8pt}}||@{\hspace{1.8pt}}c@{\hspace{1.8pt}}|@{\hspace{1.8pt}}c@{\hspace{1.8pt}}|@{\hspace{1.8pt}}c@{\hspace{1.8pt}}|@{\hspace{1.8pt}}c@{\hspace{1.8pt}}|@{\hspace{1.8pt}}c@{\hspace{1.8pt}}|@{\hspace{1.8pt}}c@{\hspace{1.8pt}}|@{\hspace{1.8pt}}c@{\hspace{1.8pt}}|@{\hspace{1.8pt}}c@{\hspace{1.8pt}}|@{\hspace{1.8pt}}c@{\hspace{1.8pt}}|@{\hspace{1.8pt}}c@{\hspace{1.8pt}}|@{\hspace{1.8pt}}c@{\hspace{1.8pt}}|@{\hspace{1.8pt}}c@{\hspace{1.8pt}}|@{\hspace{1.8pt}}c@{\hspace{1.8pt}}|}
			\hline
			Method & S1 & S2 & S3 & S4 & S5 & S6 & S7 & S8 & S9 & S10 & S11 & S12 & S13  \\
			\hline		&&&&&&&&&&&&&\vspace{-0.18cm} \\ 
			\hline
			PTM & 42.45  &39.34 &37.30 &38.46 &35.75   &39.36 & 38.11&  40.26 & 36.68 & 35.59 & 34.29  & 34.11  & 34.34\\
			HSH & 27.58 & 29.96 & 30.00 & 24.35 & 28.89 &   26.37 & 28.02& 30.02 & 24.50 &  28.30 & 30.67  & 31.25  & 33.90\\
			LDR & 43.10 & 36.29 & 32.72 & 35.76 & 31.54 &   35.66 &34.00 & 34.82 & 33.58 &  30.53 & 29.19  & 31.11  & 33.61 \\
			ALR\_H &45.71  & 41.56& 38.53 & 41.97 & 39.89&  44.44& 39.26 & {\bf 41.07} & 38.77 &39.98 & {\bf 38.12} &   {\bf 44.03} &   {\bf 46.64}\\
			ALR\_R &{\bf 47.39}    &   {\bf 42.04} &    {\bf 39.07}  &  {\bf 44.08} &    {\bf 42.76} &    {\bf 44.91}  &   {\bf 40.34} &  40.70  &  {\bf 41.69}  &   {\bf 40.29} & NA& NA& NA\\	
			\hline
		\end{tabular}
	\end{center}
\end{table}

\begin{table}\vspace{-0.2cm}
	\fontsize{5.6pt}{6.75pt}\selectfont
	\caption{\small Average SSIMs for 13 scenes.} \vspace{-0.5cm} \label{table:SSIMvalue}
	\begin{center}
		\begin{tabular}{@{}|@{\hspace{1pt}}c@{\hspace{1pt}}||@{\hspace{1pt}}c@{\hspace{1pt}}|@{\hspace{1pt}}c@{\hspace{1pt}}|@{\hspace{1pt}}c@{\hspace{1pt}}|@{\hspace{1pt}}c@{\hspace{1pt}}|@{\hspace{1pt}}c@{\hspace{1pt}}|@{\hspace{1pt}}c@{\hspace{1pt}}|@{\hspace{1pt}}c@{\hspace{1pt}}|@{\hspace{1pt}}c@{\hspace{1pt}}|@{\hspace{1pt}}c@{\hspace{1pt}}|@{\hspace{1pt}}c@{\hspace{1pt}}|@{\hspace{1pt}}c@{\hspace{1pt}}|@{\hspace{1pt}}c@{\hspace{1pt}}|@{\hspace{1pt}}c@{\hspace{1pt}}|}
			\hline
			Method & S1 & S2 & S3 & S4 & S5 & S6 & S7 & S8 & S9 & S10 & S11 & S12 & S13  \\
			\hline		&&&&&&&&&&&&&\vspace{-0.18cm} \\ 
			\hline
			PTM & 0.9841  & 0.9788 & 0.9803&  0.9745 & 0.9738 & 0.9710 &  0.9819 & 0.9793&  0.9787  & 0.9756 & 0.9585  &  0.9618  &  0.9572\\
			HSH & 0.7469  &  0.7467&  0.7758& 0.4727&  0.7212&   0.3762 &  0.6407 &  0.6982 & 0.6735 &  0.8020 &0.9215   & 0.9044  &  0.9337\\
			LDR & 0.9800  & 0.9601 &  0.9606 &  0.9649  & 0.9560& 0.9473 & 0.9704 &0.9561 & 0.9610   & 0.9546  & 0.9219  &  0.9294  &  0.9386\\
			ALR\_H &0.9914 &  0.9913&  0.9939 & 0.9885 & 0.9924&  0.9905&  0.9913  & {\bf 0.9916} & 0.9879  &  0.9922 & {\bf 0.9924} &   {\bf 0.9930} & {\bf 0.9936}\\
			ALR\_R &{\bf 0.9918}    &   {\bf 0.9918} &    {\bf 0.9945}  &  {\bf 0.9903} &    {\bf 0.9940} &    {\bf 0.9906}  &   {\bf 0.9918} &  0.9915  &  {\bf 0.9916} &   {\bf 0.9925} & NA& NA& NA\\	
			\hline
		\end{tabular}
	\end{center}
\end{table}

\begin{table}\vspace{-0.2cm}
	\fontsize{5.6pt}{6.75pt}\selectfont
	\caption{\small Average MS-SSIMs for 13 scenes.} \vspace{-0.5cm} \label{table:MSSSIMvalue}
	\begin{center}
		\begin{tabular}{@{}|@{\hspace{1pt}}c@{\hspace{1pt}}||@{\hspace{1pt}}c@{\hspace{1pt}}|@{\hspace{1pt}}c@{\hspace{1pt}}|@{\hspace{1pt}}c@{\hspace{1pt}}|@{\hspace{1pt}}c@{\hspace{1pt}}|@{\hspace{1pt}}c@{\hspace{1pt}}|@{\hspace{1pt}}c@{\hspace{1pt}}|@{\hspace{1pt}}c@{\hspace{1pt}}|@{\hspace{1pt}}c@{\hspace{1pt}}|@{\hspace{1pt}}c@{\hspace{1pt}}|@{\hspace{1pt}}c@{\hspace{1pt}}|@{\hspace{1pt}}c@{\hspace{1pt}}|@{\hspace{1pt}}c@{\hspace{1pt}}|@{\hspace{1pt}}c@{\hspace{1pt}}|}
			\hline
			Method & S1 & S2 & S3 & S4 & S5 & S6 & S7 & S8 & S9 & S10 & S11 & S12 & S13  \\
			\hline		&&&&&&&&&&&&&\vspace{-0.18cm} \\ 
			\hline
			PTM & 0.9955 &  0.9942 & 0.9954 &  0.9908  &0.9915 & 0.9918&  0.9954&  0.9946&  0.9900  & 0.9893 & 0.9891  &  0.9821  &  0.9859\\
			HSH & 0.9718 &  0.9818 &  0.9799  & 0.9152 &  0.9781&   0.9266  & 0.9737  & 0.9661 & 0.9608 & 0.9765 &0.9765   & 0.9682  &  0.9836\\
			LDR & 0.9942 &  0.9878 & 0.9875 & 0.9856 & 0.9820&  0.9846&  0.9907 & 0.9872 & 0.9793 &  0.9754 &0.9680  &  0.9646   & 0.9796\\
			ALR\_H &0.9958 & 0.9959 & 0.9970 &  0.9946 & 0.9965& 0.9954 & 0.9957 & {\bf  0.9958} & 0.9939 & 0.9962 & {\bf 0.9962} &   {\bf 0.9965} &   {\bf 0.9968}\\
			ALR\_R &{\bf 0.9961}   &   {\bf 0.9961} &    {\bf 0.9972}  &  {\bf 0.9953} &    {\bf 0.9973} &    {\bf 0.9955}  &   {\bf 0.9959} &  0.9957  &  {\bf 0.9959} &   {\bf 0.9963}  & NA& NA& NA\\	
			\hline
		\end{tabular}
	\end{center}
\end{table}

Fig.~\ref{fig:vscompar-npl} shows the recurrence results using NPL for several typical scenes S2 (near-Lambertian), S4 (transparent), S6 (cast shadow) and S9 (specular), including the zoom-in regions and corresponding 10-times magnified absolution difference images. Similarity, Fig.~\ref{fig:vscompar-nsl} shows the recurrence results for S11 (near-Lambertian) and S12 (transparent \& specular) using sNSL. We can see that our method ALR\_H or ALR\_R can generate more accurate lighting recurrence results than baselines for both quantitative comparison and visual perception. Furthermore, see Fig.~\ref{fig:vscompar-npl} and Fig.~\ref{fig:vscompar-nsl}, although PTM and LDR can also achieve not bad recurrence results for the near-Lambertian scenes, they cannot do well for the non-Lambertian scenes. This is because the non-Lambertian regions are usually more hard to reproduce than near-Lambertian regions for these SLR methods. Thanks to the physical lighting recurrence, although our ALR method is based on Lambertian assumption, it can still achieve excellent lighting recurrence accuracy for non-Lambertian scenes.

\begin{figure*}[t]
	\begin{center}
		\includegraphics[width=0.98\linewidth]{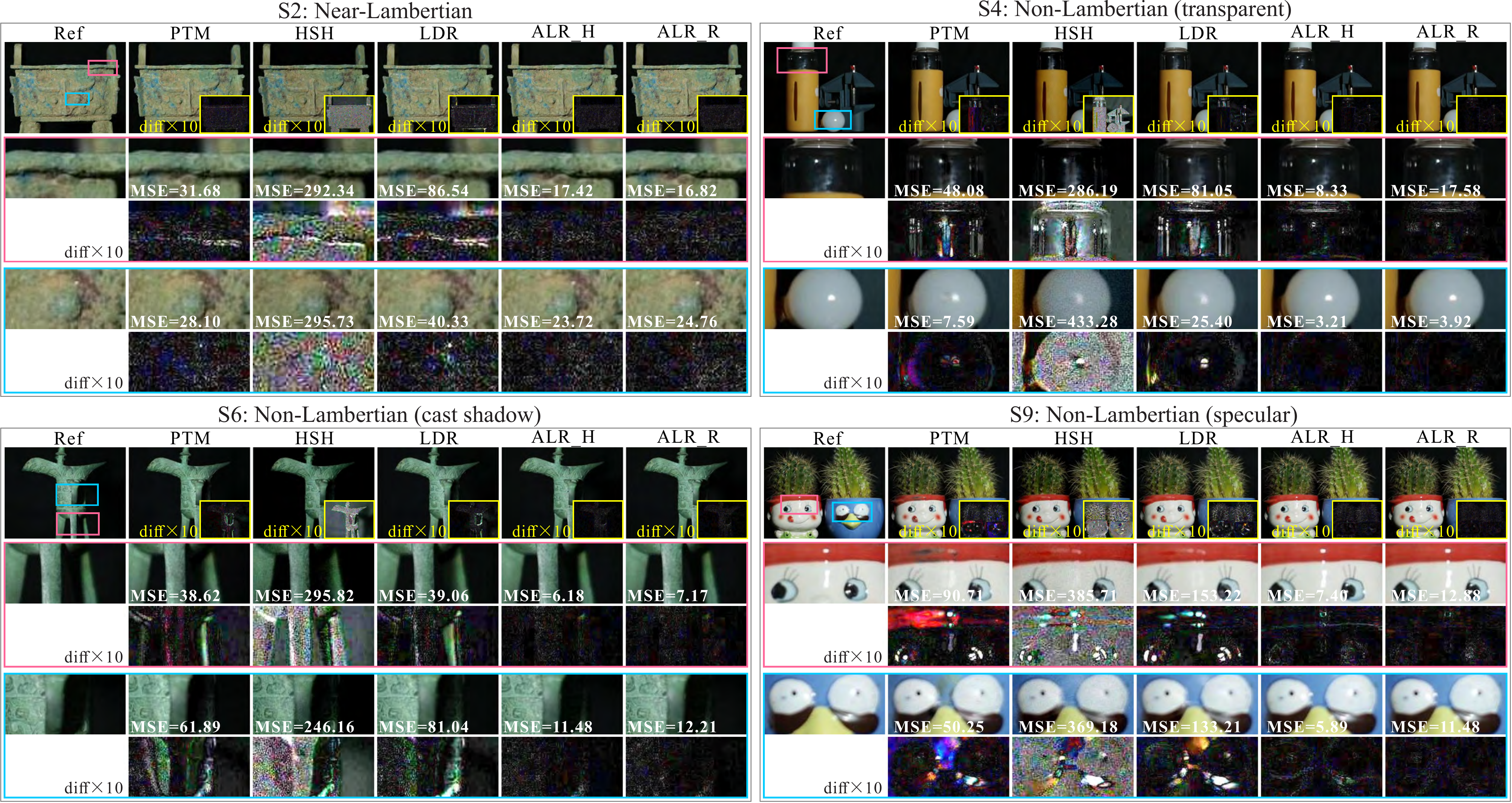} \vspace{-0.4cm}
	\end{center}
	\caption{ \small { Some results of our ALR method and 3 baselines for one near-Lambertian scene (S2) and 3 non-Lambertian scenes (S4, S6 and S9). For each scene, the first row shows the reference image and the lighting recurrence images of 4 methods. The 2--5 rows show the zoom-in regions and corresponding difference images (magnified by 10). Besides, the MSEs are also provided.}}
	\label{fig:vscompar-npl}\vspace{-0cm}
\end{figure*}

\begin{figure*}[t]
	\begin{center}
		\includegraphics[width=0.98\linewidth]{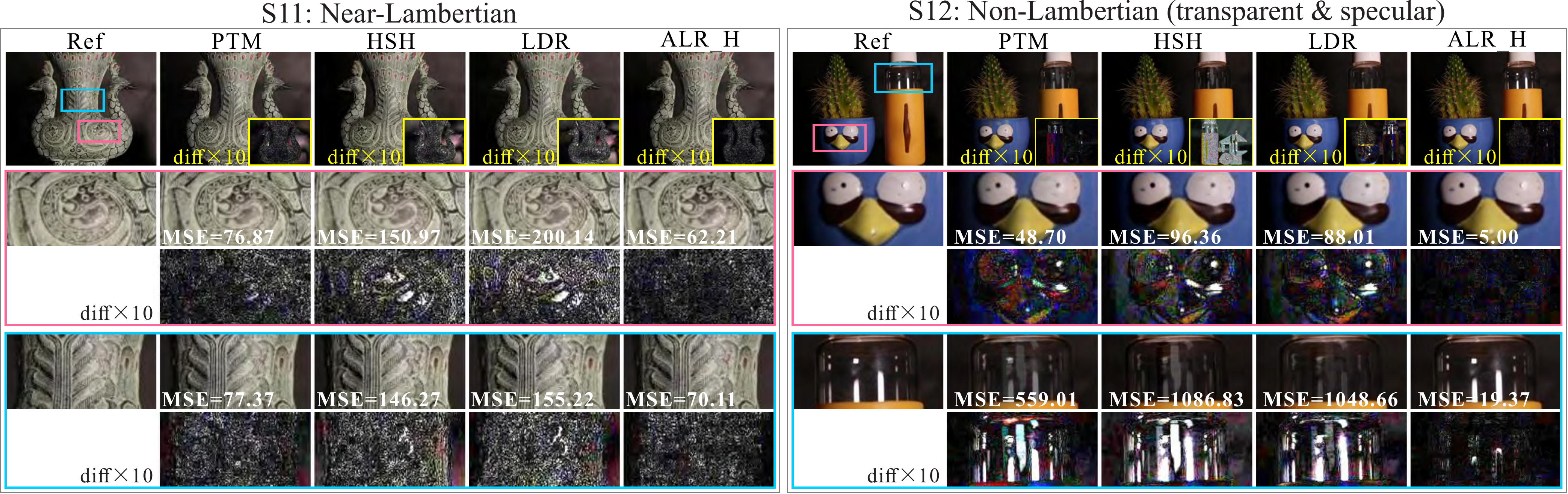} \vspace{-0.4cm}
	\end{center}
	\caption{ \small { Some results of the lighting recurrence methods with small near surface light source.} }
	\label{fig:vscompar-nsl}
\end{figure*}


\subsection{Speed comparison}
The proposed ALR method can generate real-time navigation feedback for adjusting light source. Table~\ref{table:speed} shows the FPSs of our ALR method under different image sizes. Since the cannon 5D Mark III camera we used cannot acquire the real-time original image (5760$\times$3840), we just use the live view image (960$\times$640) and perform downsampling by 2 (480$\times$320) to support light source navigation in our experiments. After ALR, we capture the 5760$\times$3840 image as the finial lighting recurrence image.
Besides, We record the average time of ALR with hand (ALR\_H) and robotic arm (ALR\_R). It takes about $70$s for ALR\_H and $30$s for ALR\_R in our experiments (exclude the time of image capturing and initialization). For ALR\_H, we can quickly adjust the light source to a local range of the target but it is not easy to achieve a better recurrence accuracy because of the instability of manual adjustment. In contrary, since the robotic arm cannot percept the scene, ALR can only rely on the bisection adjustment strategy to adjust light source, so the light source cannot be rapidly adjusted to approach the target in the beginning. But as the light source approaches the target, the advantage of movement stability of robotic arm can help it to achieve higher accuracy.

\begin{table}[!htb] \vspace{-0cm}
	\scriptsize
	\caption{\small FPSs of our ALR method under difference image sizes.}  \vspace{-0.5cm} \label{table:speed}
	\begin{center}
		\begin{tabular}{|c||c|c|c|c|c|} \hline
			Size (pixel)  & 96$\times$64 & 192$\times$128  & 288$\times$192 & 384$\times$256 & 480$\times$320 \\
			\hline		&&&&&\vspace{-0.22cm} \\ 
			FPS & 33.0961 & 27.2844 & 18.2382 & 13.7874 & 9.7125 \\ 
			\hline 
			\hline
			Size (pixel)  & 576$\times$384 & 672$\times$448 & 768$\times$512 & 864$\times$576 & 960$\times$640\\
			\hline		&&&&&\vspace{-0.22cm} \\ 
			FPS  &	7.5620 & 5.8763 & 4.6747 & 3.7234 & 3.0444\\
			\hline
		\end{tabular}
	\end{center}
\end{table}

\subsection{ALR under multiple light sources}
Up to now, what we have discussed is the ALR under single light source. In fact, using multiple light sources can effectively achieve higher imaging quality and capture more rich scene microstructures than single light source. Except for fine-grained change monitoring and measurement, ALR under multiple light sources is also useful in many other fields, e.g., lighting re-setup for photography, which usually needs to reproduce a specific light source combination to express some character emotion or photography style~\cite{begleiter201450}. 
Fortunately, without any auxiliary, our ALR method can be directly applied to multiple light sources condition. Specifically, in the reference observation, we successively lighten each light source and capture the corresponding image. Then, during current observation, we also successively conduct ALR for each light source, using the corresponding image as the reference. Note, for each ALR, the light sources which have been relocated are regarded as the environment lighting. To verify the effectiveness, we conduct ALR under two near point light sources by two robotic arms for scene S3--5. The results are shown in Fig.~\ref{fig:ALRtwoLight}. For each case, the first row indicates the two reference images, the second row denotes the ALR results. The recurred image and the 10-times magnified absolute differences are also shown in Fig.~\ref{fig:ALRtwoLight}. We can find the ALR results faithfully reproduce the scene surface details of the reference one.

\begin{figure}[!htb]
	\begin{center}
		\includegraphics[width=1\linewidth]{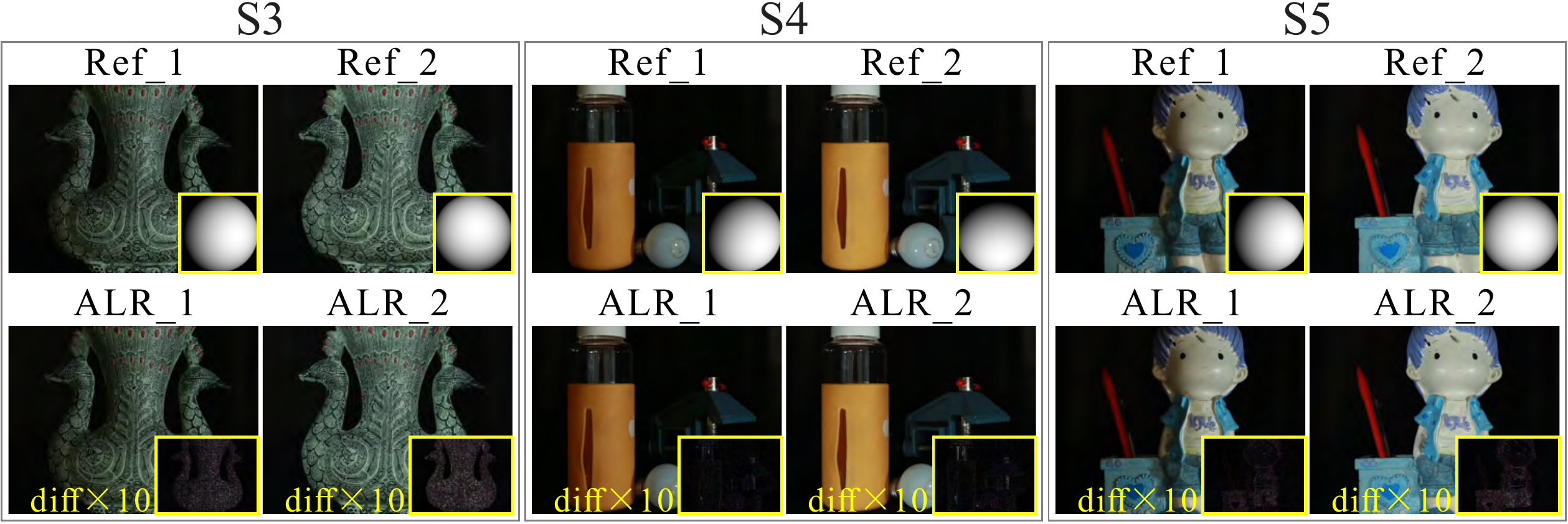} \vspace{-0.8cm}
	\end{center}
	\caption{\small {ALR results under two near point light sources.}} \vspace{-0cm}
	\label{fig:ALRtwoLight} \vspace{-0cm}
\end{figure}

\section{Real-World Applications}

\begin{figure*}[!htb]
	\begin{center}
		\includegraphics[width=1.0\linewidth]{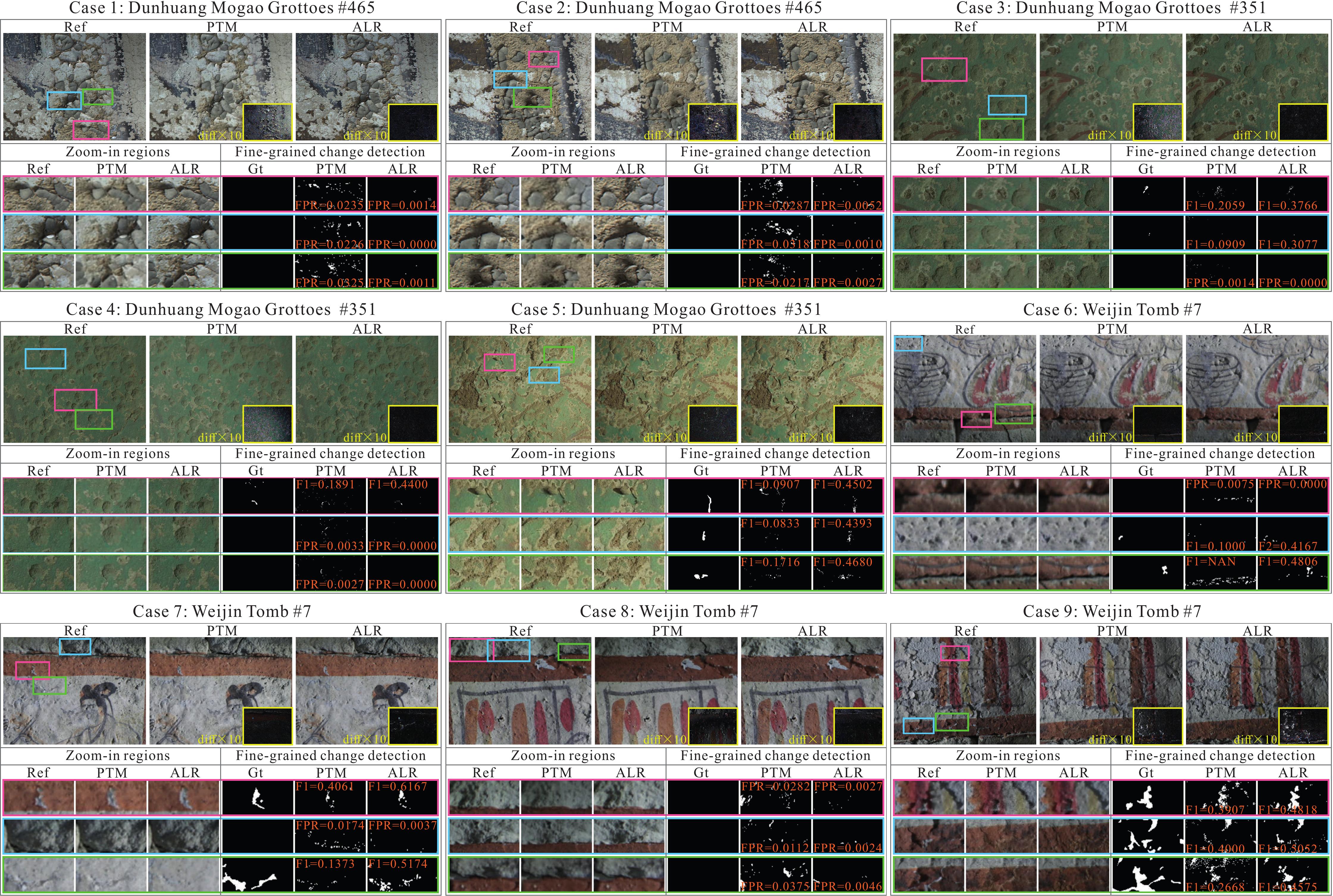} \vspace{-0.6cm}
	\end{center}
	\caption{\small {Some results of ALR and PTM methods on fine-grained change detection of ancient murals in Dunhuang Mogao Grottoes \#465 (Case 1--2), \#351 (Case 3--5) and Weijin Tomb \#7 (Case 6--9). We use 20 images under different lighting conditions to conduct PTM method. For each case, the zoom-in regions and corresponding change detection results are shown in the bottom, including the FPR or F1-Measure values.}}
	\label{fig:changedetection}\vspace{-0cm}
\end{figure*}

One promising application of the proposed ALR method is accurately capturing fine-grained change of high-value objects. For instance, in cultural heritage conservation, an essential problem is to detect and measure the minute changes of cultural relics from two sets of observations within long time intervals. This is an open real-world problem since cultural relics (e.g., Dunhuang Mogao murals) suffer from various of deteriorations even though they have been seriously protected. However, this is a quite tough problem, and challenges particular from misaligned image position as well as misaligned lighting condition. In addition, some changes can only be clearly observed in specific lighting condition so that ALR is of great importance in fine-grained change detection problem.

To address this problem, we combine our ALR method with the high-accurate actively camera relocalization method (ACR)~\cite{our2018ACR} and a state-of-the-art fine-grained change detection method (FGCD)~\cite{fcd:feng2015fine}. We apply our ALR to actively capture and measure the fine-grained changes of ancient murals in two World Cultural Heritage Sites, Dunhuang Mogao Grottoes (Case 1--5) and Weijin Graves (Case 6--9) in Fig.~\ref{fig:changedetection}. Specifically, given the reference image, we first relocalize current camera pose via ACR, then we do ALR to physically reproduce the lighting condition of reference image and take current images for evaluation. The time interval between twice observations is one year for Case 1--2 and Case 6--9, and 12 days for Case 3--5. Besides, we also capture 20 images under different lighting conditions in current observation and carry on PTM~\cite{RTI:PTM01} to generate relighting images for comparison. Fine-grained changes are detected by FGCD algorithm for both ALR and PTM results. 

All the results are shown in Fig.~\ref{fig:changedetection}. Refer to the zoom-in regions and corresponding FGCD results, we can clearly see that our ALR can generate much higher F1-Measure and lower FPR errors. This is because that some surface details cannot be faithfully reproduced by PTM, e.g., the cast shadow region of Case 2 in Fig.~\ref{fig:changedetection}. Note, for some scenes, e.g., Case 1, there is no change between twice observations, so we show the FPR value instead of F1-Measure. In a word, our ALR supports much more accurate fine-grained change detection.

\section{Conclusion}

In this paper, we have studied a new problem, active lighting recurrence (ALR), that aims to actively reproduce the lighting condition of a single reference image. To achieve instant and accurate ALR guidance, we propose a simple yet effective analogy parallel lighting (apl) based ALR approach. We show that the proposed approach works well for the commonly-used realistic near point light source and small near surface light source, with strict theoretical equivalence and convergence guarantees. Besides, we also prove the invariance of our approach to the ambiguity of normal and lighting decomposition. To the best of our knowledge, this is the first solid ALR study in computer vision and robotics. 

ALR plays a crucial role in real-world fine-grained change detection (FGCD) tasks of cultural heritages. Different with existing synthetic lighting recurrence (SLR), our ALR guarantees the physical correctness of the recurrence image and can be conducted in real-time. Besides, our method supports both manual operation and robotic platform, i.e., has strong environmental adaptability. So far, our approach has been successfully applied to a number of real-world FGCD tasks, e.g., for the first time, we discover the minute change in less than 2 weeks time interval in Dunhuang Mogao Grottoes.

In this work, we propose the new ALR problem and achieve an effective ALR method. Our work provides a useful and novel research topic in active robotic vision, which is the significance of this paper. In fact, our work is just a beginning of ALR problem. In the future, we plan to further study ALR with different light source types and under varied environment lightings during twice observations.

%
%

\ifCLASSOPTIONcaptionsoff
  \newpage
\fi


\bibliographystyle{IEEEtran}
\bibliography{IEEEabrv,ALR_refs}
\end{document}